\documentclass[twoside,11pt]{article}

\usepackage{jmlr2e}
\usepackage{hyperref}       % hyperlinks
\usepackage{url}            % simple URL typesetting
\usepackage{booktabs}       % professional-quality tables
\usepackage{amsmath,amssymb,amsfonts}
\usepackage{nicefrac}       % compact symbols for 1/2, etc.
\usepackage{microtype}      % microtypography
\usepackage{xcolor}         % colors
\usepackage{graphicx, array}
\usepackage{tcolorbox}
\usepackage{subfig}
\usepackage{euler, latexsym, epsfig, epic}
\usepackage{algorithm, algpseudocode}
\usepackage{tikz}
\usepackage{nicematrix}
\usetikzlibrary{fit}
\usetikzlibrary{shapes,arrows}
\graphicspath{ {./figs/} }
\usepackage[toc,page]{appendix}

\renewcommand{\thesubfigure}{\alph{subfigure}}

\newcommand{\argmax}{\mathop{\rm argmax}}

\def\0{\boldsymbol{0}}

\def\Y{\boldsymbol{Y}}

\def\C{\mathcal{C}}

\def\U{\boldsymbol{U}}

\def\P{\boldsymbol{P}}

\def\I{\mathcal{I}}

\def\J{\mathcal{J}}

\def\M{\mathcal{M}}

\def\P{\mathcal{P}}

\def\U{\mathcal{U}}

\def\Y{\mathcal{Y}}

\DeclareMathOperator*{\height}{height}

\DeclareMathOperator*{\inn}{in}
\DeclareMathOperator*{\out}{out}

\DeclareMathOperator{\rank}{rank}

\DeclareMathOperator*{\Gr}{Gr}

\def\RR{\mathbb{R}}
\def\CC{\mathbb{C}}

\long\def\answer#1{}

\long\def\comment#1{}

\usepackage{xcolor}

\usepackage{enumitem}

\jmlrheading{24}{2023}{1-37}{10/23}{xx/xx}{Yunzhen Yao, Liangzu Peng, and Manolis C. Tsakiris}

\ShortHeadings{Unlabeled Principal Component Analysis and Matrix Completion}{Yao, Peng, and Tsakiris}
\firstpageno{1}

\begin{document}

\title{Unlabeled Principal Component Analysis \\ and Matrix Completion}

\author{\name Yunzhen Yao \email yunzhen.yao@epfl.ch \\
       \addr School of Computer and Communication Sciences \\
       EPFL \\
       CH-1015 Lausanne, Switzerland
       \AND
       \name Liangzu Peng \email lpenn@seas.upenn.edu \\
       \addr 
   	  \AND
      \name Manolis C. Tsakiris \email manolis@amss.ac.cn \\
  	  \addr Key Laboratory for Mathematics Mechanization \\
  	  Academy of Mathematics and Systems Science \\
  	  Chinese Academy of Sciences \\
  	  Beijing, 100190, China
}

\editor{TBD}

\maketitle
\footnotetext{A short version of this work has been published in NeurIPS 2021 \citep{yao2021unlabeled}.} 

\begin{abstract}
We introduce robust principal component analysis from a data matrix in which the entries of its columns have been corrupted by permutations, termed Unlabeled Principal Component Analysis (UPCA). Using algebraic geometry, we establish that UPCA is a well-defined algebraic problem in the sense that the only matrices of minimal rank that agree with the given data are row-permutations of the ground-truth matrix, arising as the unique solutions of a polynomial system of equations. Further, we propose an efficient two-stage algorithmic pipeline for UPCA suitable for the practically relevant case where only a fraction of the data have been permuted. Stage-I employs outlier-robust PCA methods to estimate the ground-truth column-space. Equipped with the column-space, Stage-II applies recent methods for unlabeled sensing to restore the permuted data. Allowing for missing entries on top of permutations in UPCA leads to the problem of unlabeled matrix completion, for which we derive theory and algorithms of similar flavor. Experiments on synthetic data, face images, educational and medical records reveal the potential of our algorithms for applications such as data privatization and record linkage.
\end{abstract}


\section{Introduction}\label{section:Introduction}

In principal component analysis, a cornerstone of machine learning and data science, one is given a data matrix $\tilde{X}$,  assumed to be a corrupted version of a ground-truth data matrix $X^*=[x_1^* \cdots x_n^*] \in \RR^{m \times n}$, typically but not necessarily assumed to have low rank, and the objective is to estimate $X^*$ or the column-space $S^* \subset \RR^m$ of $X^*$. The most common types of corruptions that have attracted interest in modern studies are additive sparse perturbations \citep{candes2011robust,zhang2018robust}, outlier data points that lie away from $S^*$ \citep{xu2012robust,vaswani2018robust}, and missing entries, the latter also known as low-rank \textemdash{or even high-rank \citep{eriksson2012high,ongie2017algebraic,ongie2021tensor}\textemdash} matrix completion \citep{candes2009exact,ganti2015matrix,balzano2018streaming,eftekhari2019streaming,bertsimas2020fast}. 

Recently, permutations have been emerging as another type of data corruption, typically set in the context of linear regression, where the correspondences between the input and the output data have been partially distorted or are even entirely unavailable
\citep{Unnikrishnan-Allerton2015,Unnikrishnan-TIT18,Hsu-NIPS17,Slawski-JoS19,Slawski-JMLR2020,Zhang-ICML2020,marano2020making, Wang-TSP2020, Tsakiris-TIT2020,Mazumder-MP2023,peng2022global,onaran2022shuffled,azadkia2022linear}. There, one is given a point $x^*$ of a linear subspace $S^*$, but only up to a permutation of its coordinates, say $\tilde{x}=\Pi^* x^*$ with $\Pi^*$ an unknown permutation, and the goal is to find $x^*$ from the data $\tilde{x},S^*$. An alternative formulation for this problem is that given a matrix $A\in \mathbb{R}^{m \times r}$, which can be regarded as a basis of the linear subspace $S^*$, and a response vector $\tilde{x} = \Pi^* A c^*$ shuffled by an unknown permutation $\Pi^*$, the goal is to find    $\Pi^*$ and the regression coefficients $c^*$.  This \emph{Unlabeled Sensing} \citep{Unnikrishnan-Allerton2015,Unnikrishnan-TIT18} problem has many potential applications, e.g., record linkage \citep{Slawski-JoS19,Slawski-JMLR2020}, visual \citep{Cruz-CVPR2017,Cruz-PAMI2019} or textual \citep{Brown-1990,Schmaltz-2016,Shen-NeurIPS2017} permutation learning, matching problems in neuroscience \citep{nejatbakhsh2021neuron} and biology \citep{Abid-Allerton2018,ma2021optimal,Xie-ICLR2021}, and DNA-based data storage \citep{Shomorony-TIT2021,Weinberger-TIT2022,Lenz-TIT2022,Ravi-JSAIT2022}.

While methods for unlabeled sensing rely on knowledge of the source subspace $S^*$, this is not always known in practice. On the other hand, data of the form $\tilde{X}=[\tilde{x}_1,\dots,\tilde{x}_n] \in \RR^{m \times n}$ with $\tilde{x}_j = \Pi_j^*x_j^*$ an unknown permutation of an unknown point $x_j^* \in S^*$, are often available, thus raising the question of whether $S^*$ can be estimated from $\tilde{X}$. An important example of this situation is record linkage \citep{Fellegi-1969,Muralidhar-JIPS2017,Antoni-2019}, where the objective is to integrate data from independent sources, $\tilde{x}_1,\dots,\tilde{x}_n\in\mathbb{R}^m$, for subsequent data analysis. Since the entries of different records $\tilde{x}_i$'s are collected separately, the data matrix $\tilde{X}$ is \emph{unlabeled} in the sense that, the entries of its $i$-th row do not necessarily correspond to the same entity. Such kind of unlabeled data $\tilde{X}$ also arise in the context of data privatization, where the data provider anonymizes the original data $X^*$ by permuting each column of $X^*$ prior to release \citep{DomingoFerrer-IS2016,He2011PermutationAI}.  Data re-identification is a concern, since companies with privacy policies, health care providers, and financial institutions may release the collected data after anonymization. Understanding the fundamental limits of re-identifying the original data $X^*$ from the released ones $\tilde{X}$ is essential for striking a balance between data privacy and data preservation \citep{abowd2019stepping}.
Applications with unlabeled data also arise in the multiple-image correspondence problem \citep{zeng2012finding,ji2014robust} in image processing and computer vision. \cite{oliveira2005optimal} showed that estimating the correspondence of points across a sequence of images of a single rigid body motion, can be expressed as a rank-minimization problem in terms of partial permutation matrices. 

\subsection{Related Work} 
In this section, we briefly review some existing work on three problems that  interconnect in this paper; that is, unlabeled sensing, robust principal component analysis with outliers, and matrix completion with outliers. Beyond them, we also mention the recent and related works of \cite{breiding2018learning,breiding2023algebraic} and \cite{tachella2022sampling}, that combine flavors from data science, inverse problems, and algebraic models. 

\subsubsection{Unlabeled Sensing}\label{subsection:us-review}
There is a large literature on application-specific problems that involve lack of correspondences, e.g. in computer vision or statistics; here we just review four recent methods for unlabeled sensing \citep{Unnikrishnan-TIT18} that will be used in this paper. Recall that in unlabeled sensing one is given a subspace $S^* \subset \RR^m$ of dimension $r$ and a point $\tilde{x}$ which is some permuted version of a point $x^* \in S^*$ and the goal is to recover $x^*$ from $S^*$ and $\tilde{x}$. A critical distinction among methods in the literature is the sparsity level $\alpha$ of the permutation, that is the ratio of coordinates that are moved by the permutation. 

The case of \textit{dense} permutations ($\alpha=1$) is extremely challenging, with existing methods only able to handle small ranks $r$. We consider two methods known to perform best in this regime. The algebraic-geometric method called AIEM in \cite{Tsakiris-TIT2020} has linear complexity in $m$, and instead concentrates its effort on solving a polynomial system of $r$ equations in $r$ variables to produce an initialization for an expectation maximization algorithm. Currently, this method is efficient for $r \le 5$ and intractable otherwise. A very different method is CCV-Min of \cite{Peng-SPL2020}, which proceeds via branch-and-bound together with concave minimization and can handle $r\leq 8$, though intractable otherwise. For a picture regarding the computational complexity of existing unlabeled sensing methods, we refer to the discussion in \cite{Peng-SPL2020}.

For \textit{sparse} permutations (small $\alpha$) we review two methods \citep{Slawski-JoS19,Slawski-JCGS2021}. The $\ell_1$-RR algorithm of \cite{Slawski-JoS19} applies an $\ell_1$ robust linear regression relaxation and it works when $\alpha\leq 0.5$. Another approach is the Pseudo-Likelihood method (PL) of \cite{Slawski-JCGS2021}, which fits a two-component mixture density for each entry of $\tilde{x}$, one accounting for fixed data and the other for permuted data. The fitting is done via a combination of hypothesis testing, reweighted least-squares, and alternating minimization; while this method works well for $\alpha\leq 0.7$, it is sensitive to the particular basis of $S^*$ used to generate $\tilde{x}$.

The unlabeled sensing problem has in fact been explored, at least theoretically, towards greater generality \citep{Unnikrishnan-Allerton2015,Unnikrishnan-TIT18,Dokmanic-SPL2019,Tsakiris-ICML2019,Peng-ACHA2021,tsakiris2023determinantal}. In one such extended setting, already present in \citep{Unnikrishnan-Allerton2015,Unnikrishnan-TIT18}, we are only given a subset of coordinates of $\tilde{x}$ (and the subspace $S^*$) and we aim to recover $x^*$. We call this problem \emph{unlabeled sensing with missing entries}. It is a more challenging problem for which very few algorithms exist; e.g. see \citet{Elhami-ICASSP17,Tsakiris-ICML2019}. 
\citet{Tsakiris-ICML2019} proposes two algorithms: Algorithm-A is based on a combination of branch-and-bound and a dynamic programming strategy; Algorithm-B is a RANSAC-style method based also on dynamic programming computation. Both algorithms perform well, if the subspace dimension $r$ is sufficiently small (e.g., $\leq 3$) and if one is given sufficiently many entries of $\tilde{x}$.

\subsubsection{Robust PCA with Outliers}\label{section:rpca-review}
PCA methods with robustness to outliers will also play a role in this paper. Among a large literature we review four state-of-the-art methods inspired by sparse \citep{you2017provable}, cosparse \citep{tsakiris2018dual} and low-rank \citep{xu2012robust} \citep{rahmani2017coherence} representations. In that context, $\tilde{X}$ can be partitioned into inlier points that lie in an unknown $r$-dimensional subspace $S^* \subset \RR^m$ and outlier points that lie away from $S^*$; the goal is to recover $S^*$ from $\tilde{X}$. 

A successor of \cite{soltanolkotabi2012geometric}, the convex method of \cite{you2017provable}, which we refer to as Self-Expr, solves a self-expressive elastic net problem so that each $\tilde{x}_j$ is expressed as an $\ell_2$-regularized sparse linear combination of the other points. Inlier points need approximately $r$ other inliers for their self-expression as opposed to about $m$ points for outliers. The self-expressive coefficients are used to define transition probabilities of a random walk on the self-representation graph and the average of the $t$-step transition probability distributions for $t=1,\dots,T$ is used as a score for inliers vs outliers, with higher scores expected for the former. Then $\hat{S}$ is taken to be the subspace spanned by the $r$ top $\tilde{x}_j$'s. 

Dual Principal Component Pursuit (DPCP) of \cite{tsakiris2015dual,tsakiris2018dual} solves a non-smooth non-convex problem for an orthonormal basis $B^*$ of the orthogonal complement of $S^*$. In contrast to other robust-PCA methods, particularly those based on convex optimization, DPCP was shown in \cite{zhu2018dual,ding2021dual} to tolerate as many outliers as the square of the number of inliers, under a relative spherically uniform distribution assumption on inliers and outliers. This assumption is certainly not true for outliers obtained by permuting the coordinates of inliers, but we will experimentally see that an even stronger property holds for the case of UPCA (Figure \ref{fig:PCA}). 

The now classical outlier pursuit method of \cite{xu2012robust}, which we refer to as OP, decomposes via convex optimization $\tilde{X}$ into the sum of a low-rank matrix, representing the inliers, and a column-sparse matrix, representing the outliers. $\hat{S}$ is obtained as the $r$th principal component subspace of the low-rank part. Finally, the Coherence Pursuit (CoP) method of \cite{rahmani2017coherence} is based on the following simple but effective principle: with $\tilde{X}_{-j}$ the matrix $\tilde{X}$ with column $j$ removed, for each $\tilde{x}_j$ one computes its coherence $\tilde{X}^\top_{-j} \tilde{x}_j$ with the rest of the points. As it turns out, inliers tend to have coherences of higher $\ell_2$-norm than outliers, and the $r$ top $\tilde{x}_j$'s are taken to span $\hat{S}$.

\subsubsection{Matrix Completion with Outliers} \label{section:mco-review}
As mentioned in the introduction, matrix completion --- or robust PCA with missing entries --- has been a well-studied problem, with numerous developed theories \citep{candes2009exact,candes2010matrix,singer2010uniqueness,eriksson2012high,balcan2019non,tsakiris2023low} and algorithms \citep{cai2010singular,keshavan2010matrix,balzano2010online,majumdar2011some,tanner2013normalized,bertsimas2020fast}; see, e.g., \cite{davenport2016overview,vaswani2018static} for a survey.   

Closely related to this paper is a more general setting of matrix completion, where the given data matrix $\tilde{X}$ not only has some entries missing but also some of its columns are outliers. This setup is more challenging, and research on it is relatively scarce. \citet{chen2011robust,chen2015matrix} considered this problem, and proposed a convex program that minimizes a combination of a nuclear norm with an $\ell_{1,2}$ norm over a matrix of variables. Their method, which we call MCO, is shown to succeed for sufficiently many inliers and observed entries. We will make use of MCO later, to solve our \textit{unlabeled matrix completion} problem.

\subsection{Contributions}
In this paper, we consider the recovery of $X^*$ from its unlabeled version $\tilde{X}$, which we term \emph{Unlabeled Principal Component Analysis} (UPCA). We take one step further and generalize UPCA into \textit{Unlabeled Matrix Completion} (UMC),  where we now need to recover $X^*$ from only a subset of entries of $\tilde{X}$. We make contributions in the following three aspects: 
\paragraph{Theoretical contributions (Section \ref{section:theory})}
\begin{enumerate}
	\item We establish that as long as $r:=\rank(X^*)<\min\{m,n\}$ and $X^*$ is \emph{generic} (see Definition \ref{def:generic}), then up to a permutation of its rows, $X^*$ is the only matrix of rank less than or equal to $r$ that is compatible with $\tilde{X}$. This asserts that UPCA is a well-posed problem, since the inherent ambiguity of whether $\tilde{X}$ comes from $X^*$ or a row-permuted version of $X^*$ is in most cases practically harmless (Sections \ref{section:UPCA formulation} and \ref{section:UPCA-well-posed}).
	\item We establish that in this basic formulation, UPCA is a purely algebraic problem, by exhibiting a polynomial system of equations parametrized by $\tilde{X}$, whose solutions are all the row-permutations of $X^*$; solving the UPCA problem amounts to obtaining one such solution (Section \ref{section:UPCA-algebraic}).
	\item We furthermore generalize our UPCA theorems for UMC, thereby obtaining results of similar ``information-theoretical'' flavor for the scenario with permuted incomplete data (Section \ref{section:UMC-theory}).
    
 \end{enumerate}
 \paragraph{Algorithmic contributions (Section \ref{section:algorithms})}
 \begin{enumerate}
	\item Inasmuch as solving the polynomial system of UPCA is in principle NP-hard, we introduce an efficient algorithmic pipeline, Algorithm \ref{algo:UPCA}, for the practically relevant case where a significant part of the data have undergone the same \textit{dominant permutation}, while the rest of the points have been permuted arbitrarily (see Section \ref{section:UPCA-dominant-theory}); in the case of record linkage this would correspond to one of the records having much larger size than the others. The first stage of the pipeline employs PCA methods with robustness to outliers \citep{xu2012robust,soltanolkotabi2012geometric,rahmani2017coherence,you2017provable,tsakiris2018dual,zhu2018dual,lerman2018fast} to produce an estimate $\hat{S}$ of $S^*$ from $\tilde{X}$; the second stage of the pipeline uses unlabeled sensing methods \citep{Slawski-JoS19,Slawski-JCGS2021,Tsakiris-TIT2020,Peng-SPL2020} to furnish an estimate $\hat{X}$ of $X^*$ from $\hat{S}$ and $\tilde{X}$ (See Algorithm \ref{algo:UPCA}). Moreover, we introduce a simple but efficient algorithm for unlabeled sensing, Algorithm \ref{algo:LSRF}, based on least-squares with recursive filtration.
    \item Our algorithmic development for UMC is parallel to that of UPCA. We start with the dominant permutation assumption and introduce a two-stage algorithmic pipeline (Algorithm \ref{algo:UMC-proj-then-perm}). The first stage detects and completes inlier columns, and then estimates the subspace $S^*$ by matrix completion with column outliers (recall Section \ref{section:mco-review}). The second stage estimates the data matrix $X^*$ by solving the problem of unlabeled sensing on the projected coordinates for each column. 
 \end{enumerate}
 \paragraph{Experimental evaluation (Section \ref{section:experiments})}
 \begin{enumerate}
	\item We assess our algorithmic pipeline for UPCA and our proposed unlabeled sensing algorithm on synthetic data (Section \ref{subsection:RPCA-experiments} and \ref{section:UPCA synthetic}), face images (Section \ref{subsection:faces}), educational and medical records (Section \ref{section:data re-id}), with encouraging results. 
    \item We also perform experiments for the proposed UMC pipeline (Section \ref{subsection:experiments-UMC}). 
\end{enumerate}

\section{Theoretical Foundations}\label{section:theory}
In this section, we formulate and study two problems, unlabeled principal component analysis (UPCA) and unlabeled matrix completion (UMC). The goal of UPCA is to recover a ground-truth rank-deficient matrix $X^*$ from its unlabeled version $\tilde{X}$. The goal of UMC is to also recover $X^*$ from $\tilde{X}$, but now $\tilde{X}$ is a partial observation of a permuted version of $X^*$. See Figure \ref{fig:upca-umc-example} for an intuitive understanding of the setup; we will formalize the settings soon.
\begin{figure}[!h]
\centering
\subfloat[Ground-truth $X^*$ ]{$
\begin{pNiceMatrix} 
x_{11} & x_{12} & x_{13} & x_{14} \\
x_{21} & x_{22} & x_{23} & x_{24} \\
x_{31} & x_{32} & x_{33} & x_{34} \\
x_{41} & x_{42} & x_{43} & x_{44} 
\end{pNiceMatrix}$
\label{fig:upca-umc-example1}
}
\hspace{1cm}
\subfloat[UPCA data matrix $\tilde{X}$]{$
\begin{pNiceMatrix} 
x_{31} & x_{22} & x_{23} & x_{44} \\
x_{11} & x_{32} & x_{43} & x_{34} \\
x_{21} & x_{42} & x_{33} & x_{14} \\
x_{41} & x_{12} & x_{13} & x_{24} 
\end{pNiceMatrix}$
\label{fig:upca-umc-example2}
}
\hspace{1cm}
\subfloat[UMC data matrix $\tilde{X}$]{$
\begin{pNiceMatrix} 
x_{31} & * & x_{23} & x_{44} \\
* & x_{32} & x_{43} & * \\
x_{11} & x_{22} & * & x_{14} \\
x_{41} & * & x_{13} & x_{34} 
\end{pNiceMatrix}$
\label{fig:upca-umc-example3}
}
\caption{ (\ref{fig:upca-umc-example1}): Ground-truth matrix $X^*$. (\ref{fig:upca-umc-example2}): Data matrix $\tilde{X}$ for UPCA, obtained by shuffling each column of $X^*$ via some unknown permutation. (\ref{fig:upca-umc-example3}): Data matrix $\tilde{X}$ for UMC, obtained via removing some entries (indicated by $*$) and shuffling every column of $X^*$. In both UPCA and UMC, we need to recover $X^*$ from data $\tilde{X}$. }
\label{fig:upca-umc-example}
\end{figure}
\vspace{-0.5cm}

\subsection{Unlabeled Principal Component Analysis} \label{section:UPCA-theory-subsec}
\subsubsection{Problem Formulation} \label{section:UPCA formulation}

Let us denote by $\P_m$ the set of all permutations of coordinates of $\RR^m$. We let $X^*=[x_1^* \cdots x_n^*] \in \RR^{m \times n}$ be our ground-truth data matrix with rank $r<\min\{m,n\}$ and column space $S^*=\C(X^*)$, and we suppose that the available data are 
\begin{align}
	\tilde{X} = [\tilde{x}_1 \cdots \tilde{x}_n] = [\Pi_1^* x_1^* \cdots \Pi_n^* x_n^*] \in \RR^{m \times n}, \label{eq:Xtilde}
\end{align} where each $\Pi_j^*\in \P_m$ is an unknown permutation. Let $\P_m^n = \prod_{i \in [n]} \P_m$ be $n$ ordered copies of $\P_m$, where $[n]=\{1,\dots,n\}$. For $\underline{\pi} = (\Pi_1,\dots,\Pi_n) \in \P_m^n$ we set $\underline{\pi}(\tilde{X}) = [\Pi_1 \tilde{x}_1 \cdots \Pi_n \tilde{x}_n]$. We pose Unlabeled Principal Component Analysis (UPCA) as the following rank minimization problem:
\begin{align}
	\min_{\underline{\pi} \in \mathcal{P}_m^n} \, \, \rank \underline{\pi}(\tilde{X}) \label{eq:UPCA}
\end{align}  

First, note that \eqref{eq:UPCA} never has a unique solution, because if $\underline{\pi} = (\Pi_1,\dots,\Pi_n)$ is a solution, then so is $\underline{\pi}' = (\Pi \Pi_1,\dots,\Pi\Pi_n)$ for every permutation $\Pi \in \P_m$. This reveals an inherent ambiguity of UPCA: it is only possible to recover $X^*$ from $\tilde{X}$ up to a permutation $\Pi X^*$ of its rows. On the other hand, this is rather harmless in many situations, since $\Pi X^*$ is the same dataset as $X^*$ except that the row-features appear now in some different order. Thus, our hope in formulating \eqref{eq:UPCA} is that the only solutions are of the form $\underline{\pi}=(\Pi {\Pi_1^*}^{\top},\dots,\Pi {\Pi_n^*}^{\top})$ with $\Pi$ ranging across $\P_m$ and $\Pi_j^*$ as in \eqref{eq:Xtilde}.   However, without any other assumptions on the data $X^*$, there could in principle be additional undesired permutations that also give $\rank X^*$, or even worse, the minimum rank in \eqref{eq:UPCA} could be lower than $r=\rank X^*$. Our results show that for \emph{generic} enough data, such pathological situations do not occur, and the only solutions to \eqref{eq:UPCA} are the ones associated with row-permutations of $X^*$. 

\subsubsection{Elements of Algebraic Geometry} \label{section:elements of AG}

Before stating our results, we make the notion of \emph{generic} precise using some basic algebraic geometry \citep{cox2013ideals,harris2013algebraic}. Let $Z=(z_{ij})$ be an $m \times n$ matrix of variables $z_{ij}$ and $\RR[Z] = \RR\big[z_{ij}: \, i \in [m], j \in [n]\big]$ the ring of polynomials in the $z_{ij}$'s with real coefficients. An \emph{algebraic variety} of $\RR^{m \times n}$ is the set of solutions of a polynomial system of equations in $\RR[Z]$. In particular, the set of $(r+1)\times (r+1)$ determinants of $Z$ are polynomials in $z_{ij}$'s of degree $r+1$ and define the algebraic variety $$\M_r= \left. \right\{X\in \RR^{m\times n}|\rank X \le r\left. \right\},$$ since $\rank X \le r$ if and only if all $(r+1) \times (r+1)$ determinants of $X$ are zero. 

The algebraic variety $\M_r$ admits a topology, called \emph{Zariski topology}, which makes it convenient to work with. The closed sets in this topology are the \emph{algebraic subvarieties} of $\M_r$. These are sets of matrices of rank $\le r$, which in addition satisfy certain other polynomial equations in $\RR[Z]$. For example, the set of matrices of rank at most $r-1$ is a proper closed subset of $\M_r$, because in addition to the equations defining $\M_r$, it is further defined by requiring all $r \times r$ determinants to be zero. \emph{Open sets} in $\M_r$ are defined as complements of closed sets, or equivalently they are defined by requiring that certain sets of polynomials are not all simultaneously zero. For example, the set of matrices of rank exactly equal to $r$ is a proper open subset of $\M_r$ defined by the non-simultaneous vanishing of all $r \times r$ determinants of $Z$; a matrix has rank $r$ if and only if all $(r+1) \times (r+1)$ determinants are zero and least one $r \times r$ determinant is non-zero. Now, the algebraic variety $\M_r$ is \emph{irreducible} in the sense that it can not be described as the union of two proper algebraic subvarieties of it \citep{kleiman1971geometry}. A consequence of this is that non-empty open sets of $\M_r$ have the very important property of being topologically dense. This means that given a non-empty open set $\U \subset \M_r$ and a point $X \in \M_r$, every neighborhood of $X$ intersects $\U$. It follows that under any non-degenerate continuous probability measure on $\M_r$, a non-empty Zariski-open set of $\M_r$ has measure $1$. For example, the set of matrices in $\M_r$ of rank $r$ is non-empty and open, and thus it is dense. Hence a randomly sampled matrix in $\M_r$ under a continuous probability measure will have rank $r$ with probability $1$. We refer to such a fact by saying that a generic matrix in $\M_r$ has rank $r$. More generally:

\begin{definition}\label{def:generic}
	We say that a \emph{generic} matrix in $\M_r$ satisfies a property, if the property is true for every matrix in a non-empty open subset of $\M_r$. 
\end{definition}

\subsubsection{UPCA is a Well-Posed Problem} \label{section:UPCA-well-posed}
Our first theoretical result is the following:
\begin{theorem}\label{thm:main}
	For $X^*$ a generic matrix in $\M_r$, we have that $\rank \underline{\pi}(\tilde{X})  \ge r$ for any $\underline{\pi} \in \P_m^n$, with equality if and only if $\underline{\pi}(\tilde{X}) = \Pi X^*$ for some $\Pi \in \P_m$.  
\end{theorem} 

Theorem \ref{thm:main} says that for $X^*\in \M_r$ generic, and up to a permutation of the coordinates of $\RR^m$, $S^*$ is the unique $r$-dimensional subspace that explains the data $\tilde{X}$ in the UPCA sense, and $r= \rank X^*$ is the minimum objective in \eqref{eq:UPCA}.

\subsubsection{UPCA is an Algebraic Problem} \label{section:UPCA-algebraic}

How can one go about solving the discrete optimization problem \eqref{eq:UPCA}? In general, brute force selection of the $\Pi_j$'s has complexity $\mathcal{O}\big((m!)^n\big)$, which is out of the question. On the other hand, problem \eqref{eq:UPCA} has a rich algebraic structure, which allows us to show that $X^*$, up to a permutation of its rows, is the unique solution to a polynomial system of equations. 

To begin with, for each $j \in [n]$ and each $\ell \in [m]$, we define the following column-symmetric polynomials of $\RR[Z]$:
\begin{align}
	\bar{p}_{\ell,j}(Z) := \sum_{i \in [m]} z_{ij}^\ell, \ \ \ \ 
	p_{\ell,j}(Z) := \bar{p}_{\ell,j}(Z) - \bar{p}_{\ell,j}(\tilde{X})  \nonumber
\end{align}  
 Note that $\bar{p}_{\ell,j}\big(\underline{\pi}(Z)\big) = \bar{p}_{\ell,j}(Z)$ for any $\underline{\pi} \in \P_m^n$ and thus $\bar{p}_{\ell,j}(\tilde{X}) = \bar{p}_{\ell,j}(X^*)$. Now let us think of $X \in \M_r$ as a product of two matrices of size $m \times r$ and $r \times n$, and let us define another polynomial ring with variables associated to these two factors. For $i = r+1,\dots,m$, and $k \in [r]$ and $j \in [n]$, we let $b_{ik}, c_{kj}$ be a new set of variables over $\RR$. Organize the $b_{ik}$'s to occupy the $(m-r) \times r$ bottom block of an $m \times r$ matrix $B$ whose top $r \times r$ block is the identity matrix of size $r$, and the $c_{kj}$'s into a $r \times n$ matrix $C=(c_{kj})$. For $i \in [m]$, we write $b_i^\top$ for the $i$-th row of $B$; for $j \in [n]$, we write $c_j$ for the $j$-th column of $C$. With $\tilde{x}_{ij}, x_{ij}^*$ the $i$-th coordinate of $\tilde{x}_j, x_j^*$ respectively, we obtain polynomials $q_{\ell,j}$ for $\ell \in [m], \, j \in [n]$ of $\RR[B,C]$ by substituting $z_{ij} \mapsto b_i^\top c_j$ in the $p_{\ell,j}(Z)$'s above: 
\begin{align*}
	q_{\ell,j}(B,C) &:= \bar{p}_{\ell,j}(BC) - \bar{p}_{\ell,j}(\tilde{X}) \nonumber \\
	&= \sum_{i \in [m]} (b_i^\top c_j)^\ell - \sum_{i \in [m]} \tilde{x}_{ij}^\ell \nonumber \\
	&=\sum_{i \in [m]} (b_i^\top c_j)^\ell - \sum_{i \in [m]} {x^*_{ij}}^\ell  \nonumber 
\end{align*} 
The set of common roots of all $q_{\ell,j}$'s is an algebraic variety $\Y_{X^*}$ that depends only on $X^*$:
\begin{align} 
	\Y_{X^*} = & \big\{ (B',C') \in  \RR^{m \times r} \times \RR^{r \times n}\, | \, q_{\ell,j}(B',C')=0, \, \forall \ell \in [m], \, \forall j \in [n]; \, \, \, B'_{[r],[r]} = I_r \big\} \nonumber 
\end{align}
Here, $B'_{[r],[r]} = I_r$ signifies that the top $r \times r$ block of $B' \in  \RR^{m \times r}$ is the identity matrix. Then, with $\Pi \in \P_m$, if the column-space $\C(\Pi X^*)$ of $\Pi X^*$ does not drop dimension upon projection onto the first $r$ coordinates, then there exists a unique basis $B_\Pi^*$ of $\C(\Pi X^*)$ with the identity matrix occurring at the top $r \times r$ block. In that case, there is a unique factorization $\Pi X^* =  B_\Pi^* C^*_\Pi$  and the point $(B_\Pi^*,C_\Pi^*)$ lies in the variety $\Y_{X^*}$ because
\begin{align}
	q_{\ell,j}(B_\Pi^*,C_\Pi^*) &= \bar{p}_{\ell,j}(B_\Pi^*C_\Pi^*) - \bar{p}_{\ell,j}(\tilde{X})  \nonumber \\
	& = \bar{p}_{\ell,j}(\Pi X^*) - \bar{p}_{\ell,j}(X^*)  \nonumber \\
	& = \bar{p}_{\ell,j}(X^*) - \bar{p}_{\ell,j}(X^*) = 0. \nonumber 
\end{align} 
Our second result says that if $X^*$ is generic, then all points of $\Y_{X^*}$ are of this type. That is, they correspond to factorizations $B_\Pi^* C^*_\Pi$ of $\Pi X^*$ as $\Pi$ varies across all permutations: 

\begin{theorem} \label{thm:BC}
	For a generic matrix $X^*$ in $\M_r$ we have
	\begin{align} 
		\Y_{X^*} =& \big\{ (B_\Pi^*,C_\Pi^*)  \in  \RR^{m \times r} \times \RR^{r \times n} \, |  \, \Pi \in \P_m ;  \, \, B^*_{\Pi,[r],[r]} = I_r;  \, \, \Pi X^* =B_\Pi^* C_\Pi^* \big\} \nonumber 
	\end{align}
\end{theorem}

Thanks to Theorem \ref{thm:BC} we have the following important conceptual finding. Assuming $X^*$ is generic, to obtain $X^*$ up to some permutation of its rows from $\tilde{X}$, one only needs to compute an arbitrary root $(B',C')$ of the polynomial system of equations
\begin{align}
	q_{\ell,j}(B,C)=0, \, \forall \ell \in [m], \, \forall j \in [n] \label{eq:q}
\end{align} and multiply its factors to get $B'C'$. Developing a polynomial system solver for UPCA would involve two main challenges: attaining robustness to noise and scalability. We leave such an endeavor to future research.

\subsubsection{UPCA with Dominant Permutations} \label{section:UPCA-dominant-theory}

\begin{figure}[!h]
\centering
\subfloat[Ground-truth $X^*$ ]{$
\begin{pNiceMatrix} 
x_{11} & x_{12} & x_{13} & x_{14} \\
x_{21} & x_{22} & x_{23} & x_{24} \\
x_{31} & x_{32} & x_{33} & x_{34} \\
x_{41} & x_{42} & x_{43} & x_{44} 
\end{pNiceMatrix}$
\label{fig:upca-example1}
}
\hspace{1cm}
\subfloat[UPCA data matrix $\tilde{X}$]{$
	\begin{pNiceMatrix} 
		x_{31} & x_{22} & x_{23} & x_{44} \\
		x_{11} & x_{32} & x_{43} & x_{34} \\
		x_{21} & x_{42} & x_{33} & x_{14} \\
		x_{41} & x_{12} & x_{13} & x_{24} 
	\end{pNiceMatrix}$
	\label{fig:upca-example2}
}
\hspace{1cm}
\subfloat[UPCA data matrix with a dominant permutation]{$
\begin{pNiceMatrix} 
x_{41} & x_{42} & x_{23} & x_{44} \\
x_{11} & x_{12} & x_{43} & x_{14} \\
x_{21} & x_{22} & x_{33} & x_{24} \\
x_{31} & x_{32} & x_{13} & x_{34} 
\CodeAfter
  \tikz \node [draw, rounded corners, fit = (1-3) (4-3)] {};
\end{pNiceMatrix}$
\label{fig:upca-example3}
}
\caption{ (\ref{fig:upca-example1}): Ground-truth matrix $X^*$. (\ref{fig:upca-example2}): Data matrix $\tilde{X}$ for UPCA, obtained by shuffling each column of $X^*$ via some unknown permutation. (\ref{fig:upca-example3}): Data matrix $\tilde{X}$ for UPCA with a dominant permutation, obtained via shuffling some columns (columns 1, 2, 4 in the figure) by the same permutation and shuffling others arbitrarily (circled column 3).  }
\label{fig:upca-dom-example}
\end{figure}

In this section we consider a special case of interest, where part of the data have undergone the same \emph{dominant} permutation (see Figure \ref{fig:upca-dom-example}). To make this precise, we define the multiplicity $\mu(\Pi)$ of a permutation $\Pi \in \P_m$ to be the number of times that $\Pi$ appears as $\Pi= \Pi_j^*$ in \eqref{eq:Xtilde} with $j$ ranging in $[n]$. Figure \ref{fig:upca-example3} shows an example for the case $\mu(\Pi_{1}^{*}) = 3$ and $\mu(\Pi_{3}^{*}) = 1$. In fact, given the inherent ambiguity of UPCA discussed above, we may as well take this dominant permutation to be the identity matrix $I_m$ of size $m \times m$. We have:

\begin{theorem} \label{thm:dominant}
	Suppose that $\mu(I_m) \ge r+1$ while $\mu(\Pi) < r$ for any other $\Pi \neq I_m$. Then for a generic $X^* \in \M_r$, we have that $S^*$ is the unique solution to the following consensus maximization problem
	\begin{align}\label{eq:consensus-max-UPCA}
		\max_{\dim S \le r} \, \, \# \{ \tilde{x}_j \, | \, \tilde{x}_j \in S \, ; \, j \in [n]\}, 
	\end{align} where $\#$ denotes the cardinality of a set, and the maximization is taken over all subspaces $S \subset \RR^m$ of dimension $\le r$.
\end{theorem}

Theorem \ref{thm:dominant} says that for sufficiently generic ground-truth data $X^*$, the given data $\tilde{X}$ admit a natural partition into a set of inliers and outliers with respect to the linear subspace $S^*$:
\begin{align}
	\tilde{X}_{\text{in}} := \{\tilde{x}_j \, | \, \tilde{x}_j \in S^* \}, \, \, \, \, \, \, 
	\tilde{X}_{\text{out}} :=  \{\tilde{x}_j \, | \, \tilde{x}_j \not\in S^* \}  \nonumber
\end{align} Of course we do not know what the partition into inliers and outliers is, because we do not know what $S^*$ is. But the presence of this geometric structure is enough for PCA methods with robustness to outliers to operate on $\tilde{X}$ in order to estimate $S^*$. Section \ref{section:two-stage UPCA} proceeds algorithmically building on this insight.

\subsection{Unlabeled Matrix Completion}\label{section:UMC-theory}

A generalization of UPCA with practical significance is to consider PCA from data corrupted by both permutations and missing entries. To proceed we need some extra notations. 

With $\omega_j$ a subset of $[m]$ we let $P_{\omega_j} \in \RR^{m \times m}$ be the matrix representing the projection of $\RR^m$ onto the coordinates contained in $\omega_j$, that is $P_{\omega_j}$ is a diagonal matrix with the $k$th diagonal element non-zero and equal to $1$ if and only if $k \in \omega_j$. With $\omega_j$ as above for every $j \in [n]$, we write $\Omega = \bigcup_{j \in [n]} \omega_j \times \{j\} \subset [m] \times [n]$ and $\underline{p}_\Omega = (P_{\omega_1}, \dots, P_{\omega_n})$. Let $\RR^{\Omega}$ be the subspace of $\RR^{m \times n}$ of all matrices that have zeros in the complement of $\Omega$. The association $X=[x_1 \cdots x_n] \mapsto \underline{p}_\Omega (X)=[P_{\omega_1}  x_1 \cdots P_{\omega_n} x_n]$ induces a map $$\underline{p}_\Omega: \M_r \longrightarrow  \RR^{\Omega}$$ 

With this notation, in ordinary bounded-rank matrix completion (of which low-rank matrix completion is a special case) one is given a partially observed matrix $\underline{p}_\Omega(X^*)$ and the objective is to compute an at most rank-$r$ completion, that is an element of the \emph{fiber} 
$$\underline{p}_\Omega^{-1}\big( \underline{p}_\Omega(X^*) \big) = \big\{X \in \M_r \, | \, \, \underline{p}_\Omega(X) = \underline{p}_\Omega(X^*) \big\}$$ A big question is to characterize the observation patterns $\Omega$ for which $\underline{p}_\Omega(X^*)$ is generically finitely completable, in the sense that there exists a dense open set $\U$ of $\M_r$ such that for every $X^* \in \U$ the fiber $\underline{p}_\Omega^{-1}\big( \underline{p}_\Omega(X^*) \big)$ is a finite set. Even harder is the characterization of the $\Omega$'s that are generically uniquely completable, i.e. the fiber consists only of $X^*$. Both of these questions remain open in their generality, while several authors have made progress from different points of view, including rigidity theory \citep{singer2010uniqueness}, algebraic combinatorics \citep{kiraly2012combinatorial,kiraly2015algebraic}, tropical geometry \citep{bernstein2017completion} and algebraic geometry \citep{tsakiris2023low,tsakiris2020results}. 

Next, we let $\P_{\omega_j}$ be the permutations $\Pi \in \P_m$ that permute only the coordinates in $\omega_j$ and set $\P_{\Omega} = \prod_{j \in [n]} \P_{\omega_j}$.  With $\underline{\pi}_\Omega=(\Pi_1,\dots,\Pi_n) \in \P_{\Omega}$ the association $X=[x_1 \cdots x_n] \mapsto \underline{\pi}_\Omega (X)=[\Pi_1  x_1 \cdots \Pi_n x_n]$ induces a map 
$$\underline{\pi}_\Omega: \RR^{\Omega} \longrightarrow  \RR^{\Omega}$$
Now suppose that the available data matrix $\tilde{X}$ is of the form $\tilde{X} = \underline{\tilde{\pi}}_\Omega \circ \underline{p}_\Omega (X^*)$ for some $\underline{\pi}_\Omega^*  \in \P_{\Omega}$. Then the problem of \emph{unlabeled matrix completion} can be posed as finding an element $X$ in the fiber $\big(\underline{\pi}_\Omega \circ \underline{p}_\Omega\big)^{-1}(\tilde{X})$ of some map
$$\M_r \stackrel{ \underline{p}_\Omega}{\longrightarrow}  \RR^{\Omega} \stackrel{\underline{\pi}_\Omega}{\longrightarrow} \RR^{\Omega}$$ Assuming $\tilde{X}$ is generic, $\underline{p}_\Omega$ can be determined by inspection of the missing-pattern of $\tilde{X}$, while $\underline{\tilde{\pi}}_\Omega$ is unknown, as in unlabeled-PCA. Also, for a fixed $\underline{\pi}_\Omega  \in \P_{m, \Omega}$ there is no a priori guarantee that $\tilde{X}$ is in the image of the map $\underline{\pi}_\Omega \circ \underline{p}_\Omega$, while as $\underline{\pi}_\Omega$ varies in $\P_{m, \Omega}$ more than one $\underline{\pi}_\Omega \circ \underline{p}_\Omega$'s may reach $\tilde{X}$ via possibly infinitely many $X$'s in $\M_r$.

For this model, we will obtain theoretical recovery guarantees in Sections \ref{section:UMC-finite-recovery} and \ref{section:UMC-algebraic}. 
Under the dominant permutation hypothesis, we will have theoretical assertions in Section \ref{section:UMC-dominant-theory}, which  further leads us to an algorithm in Section \ref{section:UMC pipeline} and experimental analysis in Section \ref{subsection:experiments-UMC}.

\subsubsection{Finite Recovery for UMC} \label{section:UMC-finite-recovery}
In what follows we describe conditions under which finitely many $X \in \M_r$ explain the data $\tilde{X}$ for UMC. We exploit recent results of \cite{tsakiris2023low,tsakiris2020results}, where a family of generically finitely completable $\Omega$'s was studied. The following definition \citep{sturmfels1993maximal} is needed for the description of the family. 
\begin{definition}
	An $(r,m)$-SLMF (Support of a Linkage Matching Field) is a set
	$$\Phi = \bigcup_{j\in [m-r]} \varphi_{j}\times \{j\} \subset [m] \times [m-r]$$
	with the $\varphi_{j}$'s subsets of $[m]$ of cardinality $r+1$, satisfying 
	$$\# \bigcup_{j\in \mathcal{T}} \varphi_{j}\ge \# \mathcal{T} + r, \forall \mathcal{T}\subseteq [m-r].$$
\end{definition}

We have the following finiteness result for unlabeled matrix completion: 
\begin{theorem} \label{thm:main-UMC}
	Suppose $\Omega \subset [m]\times [n]$ satisfies the following two conditions. First, $\# \omega_{j}\ge r$ for every $j\in [n]$. Second, there exists a partition $[n] = \bigcup_{\nu\in [r]}\mathcal{J}_{\nu}$ of $[n]$ into $r$ subsets $\mathcal{J}_{\nu}$, such that for every $\nu\in [r]$ there exist $m-r$ subsets $\varphi_{j}^{\nu} \in \bigcup_{k\in\mathcal{J}_{\nu}}\Omega_{k}$ with $j\in [m-r]$ such that $\Phi_{\nu} = \bigcup_{j\in [m-r]} \varphi_{j}^{\nu}\times \{j\}$ is an $(r,m)$-SLMF.  For $\tilde{X} = \underline{\pi}_\Omega^* \circ \underline{p}_\Omega (X^*)$, where $X^*$ is a generic matrix in $\M_r$, and $\underline{\pi}_\Omega^* \in \P_{m, \Omega}$, the following set of unlabeled completions is finite:
	\begin{align}
		\bigcup_{\underline{\pi}_\Omega \in  \P_{m, \Omega}} \big(\underline{\pi}_\Omega \circ \underline{p}_\Omega\big)^{-1}(\tilde{X})   \label{eq:unlabeled-completions}
	\end{align}
\end{theorem} 

The set \eqref{eq:unlabeled-completions} can be thought of as the set of all at most rank-$r$ unlabeled completions of $\tilde{X}$. They can be computed, at least on a conceptual level, in a similar fashion as in UPCA by symmetric polynomials, this time supported on $\Omega$. 

\subsubsection{UMC is an Algebraic Problem} \label{section:UMC-algebraic}
We extend Theorem \ref{thm:BC} to reveal an algebraic structure of the UMC problem:
\begin{theorem} \label{thm:BC-UMC}
	Suppose $\Omega$ satisfies the hypothesis of Theorem \ref{thm:main-UMC}. Then there is a Zariski-open dense set $\U$ in $\M_r$, such that for every $X^* \in \M_r$ the unlabeled completions \eqref{eq:unlabeled-completions} of $\tilde{X}$ are of the form $B' C'$, with identity in the top $r \times r$ block of $B'$, and $(B',C')$ ranging among the finitely many roots of the polynomial system 
	$$\sum_{i \in \omega_j} (b_i^\top c_j)^\ell - \sum_{i \in \omega_j} \tilde{x}_{ij}^\ell = 0, \, \, \, \, \, \, j \in [n], \, \ell \in [\#\omega_j]$$
	In particular, for every root $(B',C')$ there is $\underline{\pi}_\Omega \in \P_{m, \Omega}$ with $\underline{\pi}_\Omega \circ \underline{p}_\Omega(B'C') = \underline{p}_\Omega(X^*)$. 
\end{theorem}

\begin{remark}\label{remark:automorphism}
	By inspecting the proof of Theorem \ref{thm:main-UMC} one sees that the effect of the permutations manifests itself only through the fact that $\P_{m, \Omega}$ is a finite group of automorphisms of $\RR^\Omega$. Hence, the proof and the statement of Theorem \ref{thm:main-UMC} remain unchanged if one replaces $\P_{m, \Omega}$ by any finite group of automorphisms of $\RR^\Omega$. 
	What will change in Theorem \ref{thm:BC-UMC}, is that one now needs to use polynomials that are invariant to the action of the specific group. Indeed, for permutations these are the symmetric polynomials. 
\end{remark}

\subsubsection{UMC with Dominant Permutations}\label{section:UMC-dominant-theory}

In Section \ref{section:UPCA-dominant-theory} we discussed UPCA under the dominant permutation assumption. In that scenario, inliers are the columns of $\tilde{X}$ that lie in the (shuffled) ground-truth subspace $\Pi S^*$, while outliers arise as the columns that are shuffled by permutations other than $\Pi$ and thus driven away from $\Pi S^*$ (Figure \ref{fig:upca-dom}); namely, inliers and outliers in UPCA are partitioned as per
\begin{align}\label{eq:UPCA-partition}
    \tilde{X}_{\text{in}} := \{\tilde{x}_j \, | \, \Pi^*_{j} = \Pi \}, \, \, \, \, \, \, 
	\tilde{X}_{\text{out}} :=  \{\tilde{x}_j \, | \, \Pi^*_{j} \ne \Pi \}. 
\end{align} We now extend that scenario to the UMC setting. 

\begin{figure}[!htbp]
\centering
\subfloat[Ground-truth $X^*$]{$
\begin{pNiceMatrix} 
x_{11} & x_{12} & x_{13} & x_{14} \\
x_{21} & x_{22} & x_{23} & x_{24} \\
x_{31} & x_{32} & x_{33} & x_{34} \\
x_{41} & x_{42} & x_{43} & x_{44} 
\end{pNiceMatrix}$
}
\hspace{1cm}
\subfloat[UPCA data matrix with a dominant permutation]{$
\begin{pNiceMatrix} 
	x_{41} & x_{32} & x_{13} & x_{44} \\
x_{11} & x_{22} & x_{43} & x_{24} \\
x_{31} & x_{12} & x_{33} & x_{14} \\
x_{21} & * & x_{23} & x_{34} 
	\CodeAfter
	\tikz \node [draw, rounded corners, fit = (1-3) (4-3)] {};
\end{pNiceMatrix}$
\label{fig:upca-dom}
}
\hspace{1cm}
\subfloat[UMC data matrix $\tilde{X}$ with a dominant permutation]{$
\begin{pNiceMatrix} 
x_{41} & x_{32} & * & x_{44} \\
x_{11} & * & x_{43} & * \\
* & x_{12} & x_{33} & x_{14} \\
x_{21} & * & x_{23} & x_{34} 
\CodeAfter
  \tikz \node [draw, rounded corners, fit = (1-3) (4-3)] {};
\end{pNiceMatrix}$
\label{fig:umc-dom}
}
\caption{Similar to Figure \ref{fig:upca-umc-example}, yet the difference is as follows. In Figure \ref{fig:upca-dom} and \ref{fig:umc-dom}, columns $1,2,4$ of $\tilde{X}$ have been shuffled by the same permutation (i.e., the dominant permutation);  column $3$ (circled) is shuffled by a different permutation and thus treated as an outlier. }
\label{fig:outliers-types}
\end{figure}

Generalizing \eqref{eq:UPCA-partition}, we can define the partition for the UMC data $\tilde{X}$:
\begin{align}\label{eq:UMC-partition}
	\tilde{X}_{\text{in}} := \{\tilde{x}_j \, | \, P_{\omega_j}\Pi^*_j = P_{\omega_j}\Pi \}, \, \, \, \, \, \, 
	\tilde{X}_{\text{out}} :=  \{\tilde{x}_j \, | \, P_{\omega_j}\Pi^*_j \ne P_{\omega_j}\Pi \}  
\end{align} 
Given the inherent ambiguity of UPCA, we can assume the dominant permutation is the identity $I_{m}$ as in Section \ref{section:UPCA-dominant-theory}  (this is equivalent to replacing the ground-truth subpsace $S^*$ by $\Pi^* S^*$, a harmless assumption for theoretical purposes, and often for practical ones as well). Then, the dominant identity permutation assumption entails sufficiently many $j$'s for which $\Pi^*_j = I_{m}$, and also that $\tilde{X}_{\text{in}}$ contains sufficiently many data points $\tilde{x}_j$'s that would span $S^*$ if correctly completed; we can thus naturally regard every point of $\tilde{X}_{\text{in}} $ as an inlier. However, pathological scenarios would arise if completing any point of $\tilde{X}_{\text{out}}$ in whatever way yielded a point in $S^*$. Fortunately, this pathological situation can in general be ruled out:
\begin{proposition}\label{prop:UMC-outlier}
	For a generic $X^*$ in $\M_r$, and for $\tilde{x}= P_{\omega} \Pi^* x^*$ a column in $\tilde{X}$ with $\# \omega \ge r+1$ and $\,P_{\omega}\Pi^* \ne P_{\omega}I_{m}$, any completion of $\tilde{x}$ is away from $S^*$, i.e. $P_{\omega} y \ne  \tilde{x}$, $\forall y\in S^*$. 
\end{proposition}
For a generic matrix $X^*\in \M_r$ with $\# \omega_j \geq r+1$ ($\forall j$), it is now safe to treat every column of $\tilde{X}_{\text{out}}$ as an outlier, and indeed \eqref{eq:UMC-partition} gives a well-defined partition of inliers $\tilde{X}_{\text{in}}$ and outliers $\tilde{X}_{\text{out}}$. This extends the insight of UPCA with the dominant identity permutation assumption, and makes it possible to estimate $S^*$ via matrix completion methods that are robust to outliers. Precise algorithmic solutions leveraging such insights fall right into Section \ref{section:UMC pipeline}.

\section{Algorithms}\label{section:algorithms}
In this section, we study the problems of UPCA and UMC under the dominant identity permutation assumption. We propose two-stage algorithmic pipelines for both problems:
\begin{itemize}
    \item For UPCA, the first stage computes a subspace $\hat{S}$ from $\tilde{X}$ via outlier-robust PCA methods. The second stage applies unlabeled sensing methods to $\tilde{X}$ (and $\hat{S}$) in a column-wise manner. Figure \ref{fig:upca_pipeline} gives a diagram, and Section \ref{section:two-stage UPCA} gives full details.
    \item The UMC pipeline parallels and extends that of UPCA. The first stage computes $\hat{S}$ from $\tilde{X}$ via matrix completion with column outliers. The second stage first takes $P_{\omega_{j}}\hat{S}$ and $P_{\omega_{j}}\tilde{x}_{j}$ as inputs, and outputs an estimate $P_{\omega_{j}}\hat{x}_{j}$ via solving the problem of unlabeled sensing with missing entries. After completing the missing entries in $P_{\omega_{j}}\hat{x}_{j}$ by a least-square method, we get the estimate $\hat{x}_{j}$. See Figure \ref{fig:umc_pipeline} and Section \ref{section:UMC pipeline}.
\end{itemize}

\begin{figure}[htp] 
	\centering
    \subfloat[UPCA]{\label{fig:upca_pipeline}
	\begin{tikzpicture}[every text node part/.style={align=center}, node distance=3cm]
		\node (Xtilde){$\tilde{X}$};
		\node (rpca) [rectangle, draw, right of=Xtilde] {fit a linear subspace $\hat{S}$ \\ via \textit{outlier-robust PCA}};	
		\node (Shat) [right of=rpca] {$\hat{S}$};
		\node (us) [rectangle, draw, right of=Shat] {apply \textit{unlabeled sensing}\\ to $\tilde{x}_{j}$ for each column $j$};
		\node (Xhat) [right of=us] {$\hat{X}$};
		
		\draw [->] (Xtilde.east) .. controls +(right:0mm) and +(left:0mm) .. (rpca.west);
		\draw [->] (rpca.east) ..  controls +(right:0mm) and +(left:0mm)  .. (Shat.west);
		
		\draw [->] (Shat.east) ..  controls +(right:0mm) and +(left:0mm)  .. (us.west);
		\draw [->] (us.east) ..  controls +(right:0mm) and +(left:0mm)  .. (Xhat.west);
	\end{tikzpicture}
    }
    \\
    \subfloat[UMC]{\label{fig:umc_pipeline}
    \begin{tikzpicture}[every text node part/.style={align=center}, node distance=3cm]
		\node (Xtilde){$\tilde{X}$};
		\node (rpca) [rectangle, draw, right of=Xtilde] {fit a linear subspace $\hat{S}$ \\ via \textit{matrix  completion} \\ \textit{with column outliers}};	
		\node (Shat) [right of=rpca] {$\hat{S}$};
		\node (us) [rectangle, draw, right of=Shat] {apply \textit{unlabeled sensing} \\ to $P_{\omega_{j}}\tilde{x}_{j}$ for each column $j$};
		\node (Xhat) [right of=us] {$\hat{X}$};
		
		\draw [->] (Xtilde.east) .. controls +(right:0mm) and +(left:0mm) .. (rpca.west);
		\draw [->] (rpca.east) ..  controls +(right:0mm) and +(left:0mm)  .. (Shat.west);
		
		\draw [->] (Shat.east) ..  controls +(right:0mm) and +(left:0mm)  .. (us.west);
		\draw [->] (us.east) ..  controls +(right:0mm) and +(left:0mm)  .. (Xhat.west);
	\end{tikzpicture}
	}
	\caption{The proposed algorithmic pipelines for UPCA and UMC. }
    \label{fig:pipelines}
\end{figure}

\subsection{Two-Stage Algorithmic Pipeline for UPCA} \label{section:two-stage UPCA}
We saw in the previous section that the UPCA problem \eqref{eq:UPCA} is well-defined (Theorem \ref{thm:main}) and in principle solvable by a polynomial system of equations (Theorem \ref{thm:BC}). However, this polynomial system is at the moment intractable to solve even for moderate dimensions. On the other hand, Theorem \ref{thm:dominant} suggests the following practical two-stage algorithmic pipeline for the case where there is a dominant permutation, which we will take to be the identity. 

\paragraph{Stage-I of UPCA}\label{section:para-pipe-I-upca} The existence of a dominant identity permutation enables estimating the underlying subspace, which is the task of Stage-I in the proposed algorithmic pipeline. Hence, at Stage-I a PCA method with robustness to outliers is employed to produce an estimate $\hat{S}$ of $S^*$ from $\tilde{X}$. Such robust PCA methods include OP \citep{xu2012robust}, Self-Repr \citep{soltanolkotabi2012geometric,you2017provable}, CoP \citep{rahmani2017coherence}, and DPCP \citep{tsakiris2018dual,zhu2018dual,lerman2018fast}, as mentioned in Section \ref{section:rpca-review}.

\paragraph{Stage-II of UPCA}\label{section:para-pipe-II-upca} Once equipped with a robust estimate of $S^*$, the aim of Stage-II is to estimate $X^*$, which can be achieved by employing methods for unlabeled sensing. These methods take a point $\tilde{x}_{j}$ of $\tilde{X}$, identified as an outlier with respect to the subspace $\hat{S}$, and return an estimate $\hat{x}_j$ of $x_j^*$ by (directly or indirectly) solving the problem
\begin{align}
\min_{\Pi \in \P_m, \, \hat{x}_{j} \in \hat{S}} \|\tilde{x}_{j} - \Pi \hat{x}_{j} \|_2 \label{eq:SLR}
\end{align}
Hence, at Stage-II of the pipeline, one feeds $\hat{S}$ and $\tilde{X}$ to an unlabeled sensing method \citep{Slawski-JoS19,Slawski-JCGS2021,Tsakiris-TIT2020,Peng-SPL2020,Mazumder-MP2023,Onaran-arXiv2022}, which operates point by point, returning for every $\tilde{x}_j$ an estimate $\hat{x}_j$. Here one may choose to threshold the $\tilde{x}_j$'s based on their distance to $\hat{S}$ and apply unlabeled sensing on the outliers only. Alternatively, if extra computational power is available for dispensing with choosing a threshold, one may apply unlabeled sensing on every $\tilde{x}_j$; we follow this approach in the experiments for UPCA. 

This proposed two-stage method is summarized in Algorithm \ref{algo:UPCA}. 
\begin{algorithm}[!htbp]
	\caption{Two-stage Algorithmic Pipeline for UPCA } \label{algo:UPCA}
	\begin{algorithmic}[1]
		\State {\bf{Input:}}{ observed data matrix $\tilde{X}$, rank $r$}
		\State estimate $\hat{S}$ of $S^*$ $\gets$ outlier-robust PCA on $\tilde{X}$ \Comment{Stage-I$\,\, $}
		\For{$j=1,\dots, n$} \Comment{Stage-II}
		\State estimate $\hat{x}_j$ of $x_j^*$ $\gets$ unlabeled sensing (\ref{eq:SLR}) on $(\tilde{x}_j, \hat{S})$ 
		\EndFor
		\State \textbf{return} estimate $\hat{X} = [\hat{x}_1, \dots, \hat{x}_n]$ of $X^*$
	\end{algorithmic}
\end{algorithm}

\paragraph{A New Method For Unlabeled Sensing: LSRF} \label{section:LSRF}
Inasmuch as there are very few scalable unlabeled sensing methods, we here propose a simple but comparatively efficient alternative named \emph{Least-Squares with Recursive Filtration} (LSRF), see Algorithm \ref{algo:LSRF}. This method is parameter-free and alternates between ordinary least-squares and a dimensionality reduction step that removes the coordinate of the ambient space on which the residual error attains its maximal value, until $r$ coordinates are left. The complexity of LSRE is $\mathcal{O}(m^2 r^2)$. 

\begin{algorithm}[!htbp]
	\caption{Unlabeled Sensing via Least-Squares with Recursive Filtration (LSRF)}\label{algo:LSRF}
	\begin{algorithmic}[1]
		\State {\bf{Input:}}{ permuted point $\tilde{x}_{j}$, basis $B^*$ of subspace $S^*$}
		\State $v^{(0)} \gets \tilde{x}_j, A^{(0)}  \gets B^*$
		\For{$k= 1, \dots, m-r$}
		\State $c \gets {A^{(k-1)} }^{\dagger} v^{(k-1)} $
		\State $i' \gets \argmax_{i} |v^{(k-1)}_{i}  - A^{(k-1)}_{i} c|$
		\State remove the $i'$th entry of $v^{(k-1)}$ to get $v^{(k)}$
		\State remove the $i'$th row of $A^{(k-1)}$ to get $A^{(k)}$
		\EndFor
		\State \textbf{return} estimate $\hat{x}_{j}=A^{(m-r)} {A^{(m-r)} }^{\dagger} v^{(m-r)}$ for $x_j^*$
	\end{algorithmic}
\end{algorithm}

\subsection{Two-Stage Algorithmic Pipeline for UMC} \label{section:UMC pipeline} 

As in UPCA, we propose a two-stage pipeline for UMC that can be effective under the dominant permutation assumption. We detail our algorithmic pipeline next. 

\paragraph{Stage-I of UMC} Stage-I of UMC parallels that of UPCA, with the same goal of estimating the ground-truth subspace $S^*$, yet with the additional challenge that the data matrix now has missing entries (Figure \ref{fig:upca-umc-example}). 
{Recall that, under the dominant permutation assumption, we can treat the columns permuted by the dominant permutation as inliers and others outliers (Figure \ref{fig:outliers-types})}. As such, Stage-I amounts to solving the problem of matrix completion with column outliers (reviewed in Section \ref{section:mco-review}). To do so, a direct solution is employing the convex program of \citet{chen2015matrix}, called MCO, which estimates $S^*$ in a way that is robust to column outliers and missing entries.

However, MCO comes with two issues that might hinder its accuracy. First, it aims to complete inliers and detect outliers \textit{simultaneously}; doing so can be very challenging  and thus error-prone. Second, its \textit{convex} program can not leverage the inherent non-convexity of the problem.  To alleviate these issues, we build upon existing \textit{non-convex} outlier-robust PCA procedures and propose an alternative to MCO. The alternative proposal detects inliers, completes inliers, and estimates the ground-truth subspace $S^*$ \textit{in cascade}: 
\begin{enumerate}[parsep=0pt] 
    \item (\textit{Detect Inliers}) We first complete all missing entries of $\tilde{X}$ by the value $0$, thereby obtaining a complete matrix $\tilde{X}_0$; such matrix $\tilde{X}_0$ is sometimes called \textit{zero-filled} data matrix \citep{Yang-ICML2015,Tsakiris-ICML2018}. Similarly to \eqref{eq:UMC-partition}, $\tilde{X}_0$ can be partitioned into \textit{zero-filled inliers} and \textit{zero-filled outliers}. Then we proceed as if the zero-filled inliers were correct completions of UMC inliers $\tilde{X}_{\text{in}}$ that span the ground-truth space $S^*$, and run existing outlier-robust PCA methods on $\tilde{X}_0$. Since the missing entries of $\tilde{X}$ are heuristically filled by zeros, such a subspace estimate might be inaccurate, away from $S^*$, and thus inadequate for the subsequent recovery task. On the other hand, if we are further given an \textit{inlier threshold} as a hyper-parameter, then we can obtain an estimated partition of inliers and outliers. In particular, an incomplete point $\tilde{x}_j$ is classified as an inlier if the distance between its zero-filled version and the estimated subspace is smaller than the inlier threshold; otherwise it is an outlier. Hence, we declare a set of inliers as determined by the partition, and will use these estimated inliers for the sequel. 
    \item (\textit{Complete Inliers}) The detected inliers $\tilde{x}_j$'s in the previous step are in fact incomplete, and at this point, we will no longer rely on their zero-filled versions; instead, we aim to find their authentic completions $(x^*_j)$'s, which are expected to span $S^*$ if the detected inliers are indeed inliers in the sense of \eqref{eq:UMC-partition}. In other words, we are now confronted with a low-rank matrix completion task. Therefore, step 2 is to complete the detected inliers using standard matrix completion algorithms.
    \item (\textit{Estimate $S^*$}) Finally, given an estimate for the ground-truth rank $r$, we can now simply perform a singular value decomposition on the matrix of completed inliers, and obtain the final estimate $\hat{S}$ of $S^*$.
\end{enumerate}
The above routine in cascade can be upgraded into a block coordinate descent method \citep{Peng-arXiv2023b}: Alternate among inlier detection and completion and subspace estimation. Moreover, convergence guarantees of \citep{Peng-arXiv2023b} might be applied here. That said, we do not pursue this idea of block coordinate descent here, as the above routine already shows satisfactory recovery performance (see Section \ref{subsection:experiments-UMC}).

\paragraph{Stage-II of UMC} After Stage-I, we have obtained an estimate $\hat{S}$ of the ground-truth subspace $S^*$ and completed inliers. Since each outlier $\tilde{x}_j$ is a point $x^*_j$ in $S^*$ except being permuted and having some entries missing, there is a chance of recovering a good estimate $\hat{x}_j$ of $x_j$ from $\hat{S}$ and $\tilde{x}_j$. 
In such cases, we can first restore the permutation of each column via solving \eqref{eq:SLR} on observed entries {(using the subspace $P_{\omega_j} \hat{S}$)} and then complete the missing entries of that column via a least-squares computation. 
We summarize our approach in Algorithm \ref{algo:UMC-proj-then-perm}. 
\begin{algorithm}[!htbp]
	\caption{Two-stage Algorithmic Pipeline for UMC} \label{algo:UMC-proj-then-perm}
	\begin{algorithmic}[1]
		\State {\bf{Input:}}{ observed data matrix $\tilde{X}$, rank $r$}
		\State $\tilde{X}_0 \gets$ fill missing entries in $\tilde{X}$ with zeros  
		\State estimate $\hat{S}$ of $S^*$ $\gets$ matrix completion with column outliers on $\tilde{X}_0$  \Comment{Stage-I$\,\, $}
		\State $\hat{B}$ $\gets$ basis of $\hat{S}$
		\For{$j=1,\dots, n$} \Comment{Stage-II}
		\State estimate $P_{\omega_{j}} \hat{\Pi}_{j}$  of $P_{\omega_{j}} \Pi_j^*$ $\gets$ unlabeled sensing \eqref{eq:SLR} on $(P_{\omega_{j}} \tilde{x}_j, P_{\omega_{j}} \hat{S})$ 
		\State estimate coefficient $\hat{c}_{j}$ of $x^*_j$ in basis $\hat{B}$ $\gets$  least squares on $(P_{\omega_{j}}  \hat{B}, P_{\omega_{j}} \hat{\Pi}_{j} \tilde{x}_j)$ 
		\State $\hat{x}_j$ $\gets$ $\hat{B} \cdot \hat{c}_{j}$ 
		\EndFor
		\State \textbf{return} estimate $\hat{X} = [\hat{x}_1, \dots, \hat{x}_n]$ of $X^*$
	\end{algorithmic}
\end{algorithm}

\color{black}

\section{Experimental Evaluation}\label{section:experiments}
Here we perform synthetic and real data experiments to evaluate the proposed algorithmic pipelines for UPCA (Section \ref{subsection:experiments-UPCA}) and UMC (Section \ref{subsection:experiments-UMC}). We use two metrics for performance evaluation. The first is the largest principal angle $\theta_{\max}(S^*,\hat{S})$ between the estimated subspace $\hat{S}$ and ground-truth $S^*$, and this is used for Stage-I to evaluate subspace learning accuracy. The second metric is the relative estimation error $\frac{\|\hat{X} - X^*\|_{F}}{\|X^*\|_{F}}$ between the estimated data matrix $\hat{X}$ and the ground-truth $X^*$, which quantifies the final performance of our algorithmic pipeline. For both metrics, smaller values imply better performance.

\color{black}
\subsection{UPCA Experiments}\label{subsection:experiments-UPCA}
We begin by assessing the performance of Stage-I of the pipeline in Section \ref{subsection:RPCA-experiments}. This entails understanding how different PCA methods with robustness to outliers behave when the outliers are induced by permutations, as in Theorem \ref{thm:dominant}.
Next in section \ref{section:UPCA synthetic}, we evaluate the overall UPCA pipeline of Algorithm \ref{algo:UPCA} on synthetic data with added spherical noise. 

\subsubsection{Stage-I of UPCA}\label{subsection:RPCA-experiments}

To understand how different PCA methods with robustness to outliers behave when the outliers are induced by permutations, we access the performance of Stage-I of the pipeline in section \ref{subsection:RPCA-experiments}. We consider Self-Expr \citep{you2017provable,soltanolkotabi2012geometric}, CoP \citep{rahmani2017coherence}, OP \citep{xu2012robust}, and DPCP \citep{tsakiris2018dual,lerman2018fast}; these methods are reviewed in Section \ref{section:rpca-review}. 

We fix $m=50$ and $n=500$. With $\dim S^*$ taking values $r=1:1:49$, $S^*$ is sampled uniformly at random from the Grassmannian $\Gr(r,m)$. Then $n$ points $x_j^*$ are sampled uniformly at random from the intersection of $S^*$ with the unit sphere of $\RR^m$ to yield $X^*$. Denote by $n_{\inn}$ the number of inliers and $n_{\out}$ the number of outliers, with $n_{\inn}+n_{\out}=n$. We consider outlier ratios $n_{\out}/n = 0.1 : 0.1 : 0.9$. For a fixed outlier ratio, we set $\tilde{\Pi}_j$ to the identity for $j \in [n_{\inn}]$ and determine the $\tilde{\Pi}_j$'s for $j > n_{\inn}$ as follows. An important parameter in the design of a permutation $\Pi$ is its sparsity level $\alpha \in [0,1]$. This is the ratio of coordinates that are moved by $\Pi$. To obtain $\tilde{\Pi}_j$, for a fixed $\alpha$, for each $j > n_{\inn}$, we randomly choose $\alpha m$ coordinates and subsequently a random permutation on those coordinates. We consider permutation sparsity levels $\alpha=1,0.6,0.2,0.1$.

In Self-Repr and CoP, $\hat{S}$ is taken to be the subspace spanned by the top $r$ $\tilde{x}_j$'s with largest inlier scores. We use the Iteratively-Reweighed-Least-Squares method proposed by \cite{tsakiris2017hyperplane} and \cite{lerman2018fast} for solving the DPCP problem. The output subspace $\hat{S}$ of OP is obtained as the $r$th principal component subspace of the decomposed low-rank matrix. 
For Self-Expr we use $\lambda = 0.95$, $\alpha = 10$ and $T=1000$, see section 5 in \cite{you2017provable}. For DPCP we use $T_{\max} = 1000$, $\epsilon = 10^{-9}$ and $\delta = 10^{-15}$, see Algorithm 2 in \cite{tsakiris2018dual}. 

Finally, OP uses $\lambda = 0.5$ and $\tau = 1$  in Algorithm 1 of \cite{xu2012robust}.

Figure \ref{fig:PCA} depicts the \emph{outlier-ratio versus rank} phase transitions, where to calibrate the analysis with what we know about these methods from prior work, we have included in the top row of the figure the phase transitions for outliers randomly chosen from the unit sphere. By reading that top row we recall: 
i) DPCP has overall the best performance across all ranks and all outlier ratios, ii) OP identifies correctly $S^*$ only in the low rank low outlier-ratio regime, as expected from its conceptual formulation, and iii) CoP and Self-Expr, even though low-rank methods in spirit, they have accuracy similar to each other and considerably better than OP. We also note that CoP is the fastest method requiring $0.51sec$ for the computation of a single phase transition plot for each trial (i.e., average time for running all settings of outlier ratio $0.1:01:0.9$ and rank $1:1:49$), Self-Expr is the slowest with $752sec$ and DPCP and OP take $1.31sec$ and $5.62sec$, respectively\footnote{Experiments are run on an Intel(R) i7-8700K, 3.7 GHz, 16GB machine. }.

\begin{figure}[!h]
\centering
\begin{tabular}{c c c c c c}
\centering
    \raisebox{\dimexpr 1cm-\height}{\quad} & {Self-Repr} & \hspace{-6mm} {CoP} & \hspace{-6mm} {OP} &  \hspace{-6mm} {DPCP} &  {\ } 
    \\  \vspace{-5mm}
    \raisebox{\dimexpr 1.4cm-\height}{\quad} & 
    \subfloat{
         \includegraphics[width=1in]{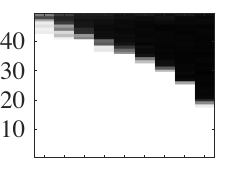}
         \label{fig:rand_seo}} & \hspace{-6mm}
    \subfloat{
         \includegraphics[width=1in]{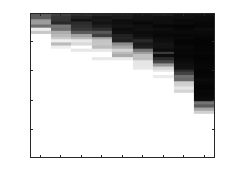}
         \label{fig:rand_cop}} & \hspace{-6mm}
    \subfloat{
         \includegraphics[width=1in]{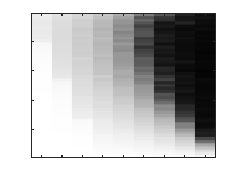}
         \label{fig:rand_l21}} & \hspace{-6mm}
    \subfloat{%
         \includegraphics[width=1in]{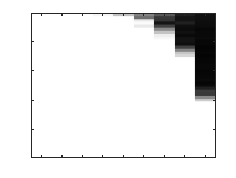}
         \label{fig:rand_irls}} & \hspace{-6mm}
     \raisebox{\dimexpr 1.5cm-\height}{\makebox[2\height]{\shortstack{random \\ outliers}}}
     \\
     \vspace{-5mm}
     \raisebox{\dimexpr 1.4cm-\height}{\quad} &
     \subfloat{%
         \includegraphics[width=1in]{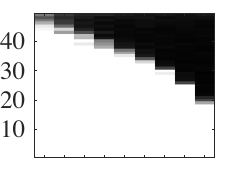}
         \label{fig:full_seo}} & \hspace{-6mm}
    \subfloat{%
         \includegraphics[width=1in]{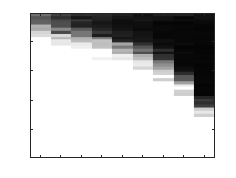}
         \label{fig:full_cop}} & \hspace{-6mm}
     \subfloat{%
         \includegraphics[width=1in]{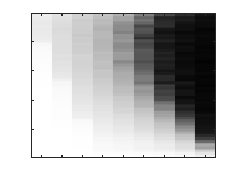}
         \label{fig:full_l21}} & \hspace{-6mm}
    \subfloat{%
         \includegraphics[width=1in]{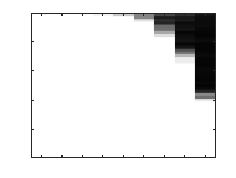}
         \label{fig:full_irls}} & \hspace{-6mm}
    \raisebox{\dimexpr 1.4cm-\height}{$\alpha=1.0$}
	\\
    \vspace{-5mm}
     \raisebox{\dimexpr 1.0cm-\height}{$r$} &
    \subfloat{%
         \includegraphics[width=1in]{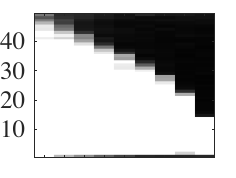}
         \label{fig:60_seo}} & \hspace{-6mm}
     \subfloat{%
         \includegraphics[width=1in]{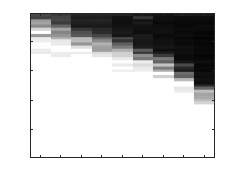}
         \label{fig:60_cop}} & \hspace{-6mm}
     \subfloat{%
         \includegraphics[width=1in]{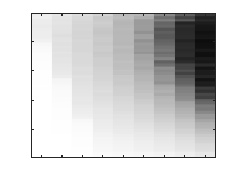}
         \label{fig:60_l21}} & \hspace{-6mm}
    \subfloat{%
         \includegraphics[width=1in]{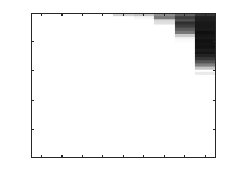}
         \label{fig:60_irls}} & \hspace{-6mm}
    \raisebox{\dimexpr 1.4cm-\height}{$\alpha=0.6$}
    \\
    \vspace{-5mm}
     \raisebox{\dimexpr 1.4cm-\height}{\quad} & 
    \subfloat{%
         \includegraphics[width=1in]{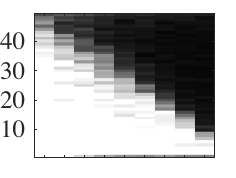}
         \label{fig:20_seo}} & \hspace{-6mm}
     \subfloat{%
         \includegraphics[width=1in]{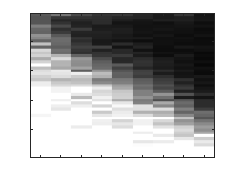}
         \label{fig:20_cop}} & \hspace{-6mm}
     \subfloat{%
         \includegraphics[width=1in]{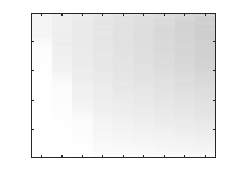}
         \label{fig:20_l21}} & \hspace{-6mm}
    \subfloat{%
         \includegraphics[width=1in]{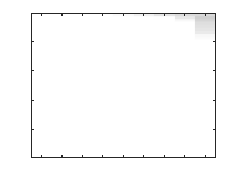}
         \label{fig:20_irls}} & \hspace{-6mm}
     \raisebox{\dimexpr 1.4cm-\height}{$\alpha=0.2$}
     \\
     \raisebox{\dimexpr 1.4cm-\height}{\quad} & 
     \subfloat{%
         \includegraphics[width=1in]{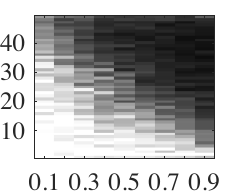}
         \label{fig:10_seo}} & \hspace{-6mm}
     \subfloat{%
         \includegraphics[width=1in]{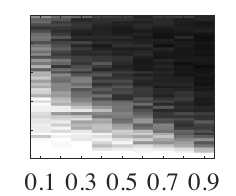}
         \label{fig:10_cop}} & \hspace{-6mm}
     \subfloat{%
         \includegraphics[width=1in]{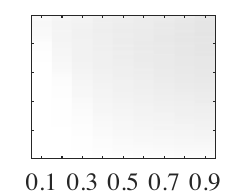}
         \label{fig:10_l21}} & \hspace{-6mm}
    \subfloat{%
         \includegraphics[width=1in]{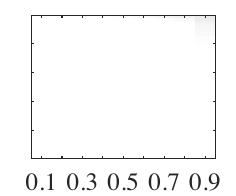}
         \label{fig:10_irls}} & \hspace{-6mm}
    \raisebox{\dimexpr 1.4cm-\height}{$\alpha=0.1$}
    \end{tabular}\\ 
{outlier ratio} 
	\caption{$\theta_{\max}(S^*,\hat{S})$ in UPCA Stage-I: outlier ratio vs. rank phase transitions for various PCA methods with robustness to outliers.}
	\label{fig:PCA}
\end{figure}

Now let us look at what happens for permutation-induced outliers. For $\alpha=1$, where the permutations move all the coordinates of the points they are corrupting, we see that the phase transition plots are practically the same as for random outliers. In other words, obtaining the outliers by randomly permuting all coordinates of inlier points, with different permutations for different outliers, seems to be yielding an outlier set as generic for the task of subspace learning as sampling the outliers randomly from the unit sphere. A second interesting phenomenon is observed when the permutation ratio is decreased to $\alpha= 0.1$. In that regime the methods exhibit two very different trends. On one hand, CoP and Self-Expr appear to break down, which is expected, because as the permutations become more sparse, the outlier points become more coherent with the rest of the data set. On the other hand, the accuracy of DPCP and OP improves for sparser permutations; a justification for this is that both methods get initialized via the SVD of $\tilde{X}$, which yields a subspace closer to $S^*$ for smaller $\alpha$. For example, the value of principal angles $\theta_{\max}(S^*,\hat{S})$ for Self-Expr, CoP, OP, DPCP, for $\alpha=0.2$, outlier ratio $0.9$ and $r=49$ are $79^\circ, 83^\circ, 14^\circ, 13^\circ$, respectively. As another example, for $\alpha=0.1$, outlier ratio $0.7$ and $r=25$ the value of $\theta_{\max}(S^*,\hat{S})$ is $67^\circ, 80^\circ, 7^\circ, (10^{-6})^\circ$, with the methods ordered as above. Overall, DPCP is consistently outperforming the rest of the methods, justifying it as our primary choice in the next section. An interesting research direction is to analyze the theoretical guarantees of these methods for this specific type of outliers.

\subsubsection{The Full Pipeline of UPCA}\label{section:UPCA synthetic}

Now we evaluate the UPCA pipeline of Algorithm \ref{algo:UPCA} on synthetic data. We keep $m=50$ as before, and add spherical noise with to a fixed SNR of $40$dB. We get the estimate $\hat{S}$ of $S^*$ via DPCP \citep{tsakiris2018dual,lerman2018fast} in Stage-I and apply the unlabeled sensing methods \citep{Tsakiris-TIT2020,Peng-SPL2020,Slawski-JoS19,Slawski-JCGS2021} and Algorithm \ref{algo:LSRF} in Stage-II to get $\hat{X}$ from $\hat{S}$ and $\tilde{X}$. We distinguish between dense and sparse permutations. 

\textbf{Dense Permutations.} We first consider dense permutations, that is $\alpha=1$. This is an extremely challenging case, with the difficulty manifesting itself through the fact that existing methods can only handle small ranks $r$. We consider AIEM and CCV-Min, two state-of-the-art methods mentioned in Section \ref{subsection:us-review}. For AIEM we use a maximum number of $1000$ iterations in the alternating minimization of \eqref{eq:SLR}. For CCV-Min we use a precision of $0.001$, the maximal number of iterations is set to $50$, and the maximum depth to $12$ for $r=3$ and $14$ for $r=4,5$. 

Figure \ref{fig:LRWC-dense} depicts the relative estimation error of $\hat{X}$ for different outlier ratios from $75\%$ ($25$ inliers) to  $94\%$ ($6$ inliers) and ranks $r=3,4,5$. To assess the overall effect of the quality of $\hat{S}$, we use two versions of AIEM and CCV-Min. The first, denoted by AIEM($\hat{S}$) and CCV-Min($\hat{S}$), uses as input the estimated subspace $\hat{S}$, while the second version, AIEM($S^*$) and CCV-Min($S^*$), uses the ground-truth subspace $S^*$. Note that the estimation error of AIEM($S^*$)/CCV-Min($S^*$) is independent of the outlier ratio. On the other hand, the estimation error of AIEM($\hat{S}$)/CCV-Min($\hat{S}$) depends on the outlier ratio through the computation of $\hat{S}$. Indeed, $\hat{S}$ is expected to be closer to $S^*$ for smaller outlier ratios, as we already know from Figure \ref{fig:PCA}. In particular, for up to $75\%$ outliers the estimation error of AIEM($\hat{S}$)/ CCV-Min($\hat{S}$) coincides with that of AIEM($S^*$)/ CCV-Min($S^*$), indicating an accurate estimation of $S^*$. At the other extreme, for $94 \%$ outliers both AIEM($\hat{S}$)/ CCV-Min($\hat{S}$) break down, indicating that the estimation of $\hat{S}$ failed. Finally, note that CCV-Min has at least half order of magnitude smaller estimation error than AIEM. This is due to our specific choice of the branch \& bound CCV-Min parameters which control the trade-off between accuracy and running time; for example, for $r=3$ and $75\%$ outliers, AIEM runs in $42msec$ with $1\%$ error, while CCV-Min needs about $15sec$ to bound $\hat{X}$ $0.42\%$ away from $X^*$. 

For AIEM, we use the customized Gr\"obner basis solvers of \cite{Tsakiris-TIT2020}, developed for $r \le 4$, which solve the polynomial system in milliseconds, and the maximum number iterations in the alternating minimization procedure is $T_{\max}=1000$. For $r=5$, the design of such solvers is an open problem\footnote{The fast solver generator of \cite{Larsson-CVPR2017} is an improved version of the one used by \cite{Tsakiris-TIT2020} for $r=3,4$. However, we found that for $r=5$ it suffers from numerical stability issues.}, thus we use the generic solver Bertini \citep{Bertini}, which runs within a few seconds. For $r \ge 6$ though, AIEM remains as of now practically intractable. 
For CCV-Min the precision is $0.001$, $T_{\max}=50$, and the maximum depth is $12$ for $r=3$ and $14$ for $r=4,5$. For $\ell_1$-RR we use $\lambda = 0.01 \sqrt{\log(n)/n}$ in (13) of \cite{Slawski-JoS19}.

\begin{figure*}[!htbp]
	\centering
	\raisebox{\dimexpr 2.2cm-\height}{{\small $\frac{\|\hat{X}-X^*\|_F}{\|X^*\|_F}$}}
	\subfloat[$r=3$]{%
		\includegraphics[width=0.27\linewidth]{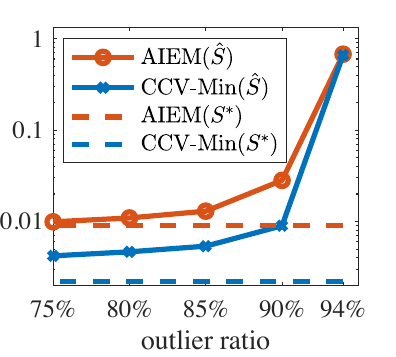}
		\label{fig:r3_10}}
	\subfloat[$r=4$]{%
		\includegraphics[width=0.27\linewidth]{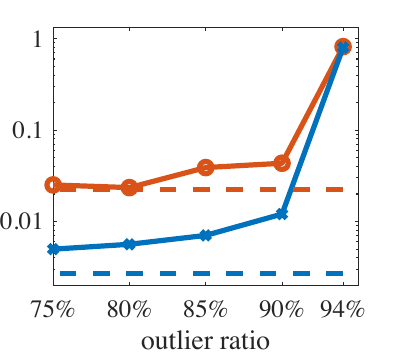}
		\label{fig:r4_10}}
	\subfloat[$r=5$]{%
		\includegraphics[width=0.27\linewidth]{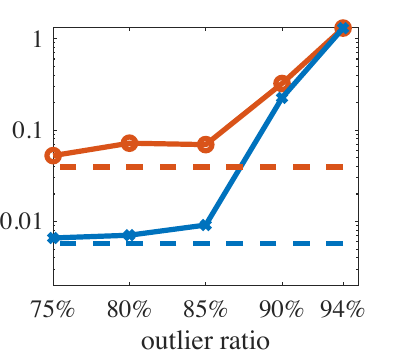}
		\label{fig:r5_10}}
	\caption{UPCA (Algorithm \ref{algo:UPCA}) for dense permutations ($\alpha=1$) with $\hat{S}$ produced by DPCP \citep{tsakiris2018dual,lerman2018fast} at Stage-I and $\hat{X}$ produced by AIEM \citep{Tsakiris-TIT2020} or CCV-Min \citep{Peng-SPL2020} at Stage-II.}
	\label{fig:LRWC-dense}
\end{figure*}

\textbf{Sparse Permutations.} The methods that we saw in the previous section certainly apply in the special case where only a fraction of the coordinates is permuted. However, they are still subject to the same computational limitations that practically require the rank $r$ to be small ($r\le 6$ for AIEM and $r \le 8$ for CCV-Min). On the other hand, the problem of linear regression without correspondences is tractable for a wider range of ranks when the permutations are sparse (small $\alpha$). This important case arises in applications such as record linkage, where domain specific algorithms are only able to guarantee partially correctly matched data. Here we consider three methods, $\ell_1$-RR \citep{Slawski-JoS19}, PL \citep{Slawski-JCGS2021}, and our proposed LSRF (Algorithm \ref{algo:LSRF}).

Figures \ref{fig:rg_10}-\ref{fig:lsr_10} show the relative estimation error of UPCA for $\alpha = 0.1:0.1:0.6$, rank $r=1:1:25$ and outlier ratio fixed to $90\%$, with $\hat{S}$ computed in Stage-I by DPCP \citep{tsakiris2018dual,lerman2018fast} and $\hat{X}$ computed via $\ell_1$-RR, PL, or LSRF from $\tilde{X}$ and $\hat{S}$ in Stage-II. It is important to note that $r=25=m/2$ is the largest rank for which unique recovery of $X^*$ is theoretically possible  \citep{Unnikrishnan-Allerton2015,Unnikrishnan-TIT18,Tsakiris-ICML2019,Dokmanic-SPL2019}. Figure \ref{fig:ldegree} shows that $\theta_{\max}(S^*,\hat{S})$ always stays below $2^\circ$, indicating the success of DPCP. As before, we also show in Figures \ref{fig:rg_gt_10}-\ref{fig:lsr_gt_10} the estimation error when $S^*$ is used instead of $\hat{S}$. Evidently, the performance is nearly identical regardless of whether $\hat{S}$ or $S^*$ is used, again justifying the success of Stage-I. Now $\ell_1$-RR and Algorithm \ref{algo:LSRF} have similar accuracy, but Algorithm \ref{algo:LSRF} is more efficient than $\ell_1$-RR, considering that computing $\hat{X}$ takes $0.3sec$ seconds for Algorithm \ref{algo:LSRF} and $1.5min$ for $\ell_1$-RR. Even though PL delivers $\hat{X}$ in $1sec$, it is not performing as well, which we attribute to its sensitivity on the particular basis of $S^*$ that is used to generate the data; this is not available here since DPCP returns the specific basis of dual principal components. 

\captionsetup[subfigure]{width=3.5cm}

\begin{figure*}[!htbp]
	\centering 
	\begin{tabular}{c c c c c}
		\centering 
		\subfloat[$\theta_{\text{max}}(S^*,\hat{S})$]{%
			\includegraphics[height= 0.20\linewidth]{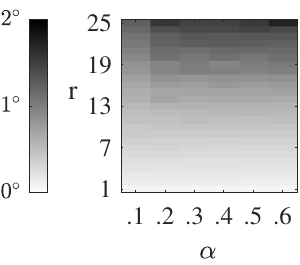}
			\label{fig:ldegree}}  & 
		{\includegraphics[height= 0.20\linewidth]{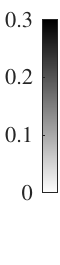}
		} &  \hspace{-0.3cm} 
		\subfloat[$\ell_1$-RR($\hat{S}$) 
		]{%
			\includegraphics[height= 0.20\linewidth]{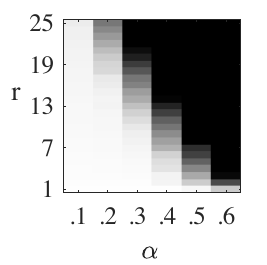}
			\label{fig:rg_10}} &  
		\hspace{-0.6cm}  
		\subfloat[PL($\hat{S}$)
		]{%
			\includegraphics[height= 0.20\linewidth]{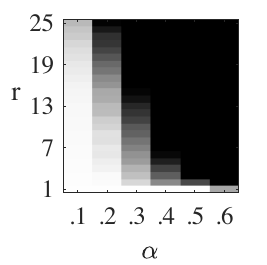}
			\label{fig:psd_10}} &  
		\hspace{-0.6cm} 
		\subfloat[LSRF($\hat{S}$)]{%
			\includegraphics[height= 0.20\linewidth]{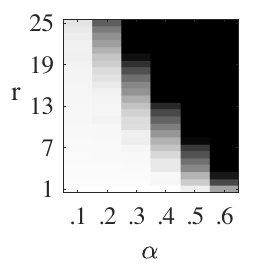}
			\label{fig:lsr_10}} \\
	& & \subfloat[$\ell_1$-RR($S^*$)]{%
         \includegraphics[height=0.20\linewidth]{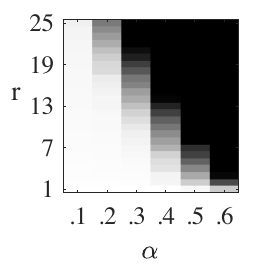}
         \label{fig:rg_gt_10}} & \hspace{-0.6cm} 
    \subfloat[PL($S^*$)]{%
         \includegraphics[height=0.20\linewidth]{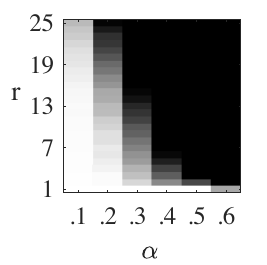}
         \label{fig:psd_gt_10}} & \hspace{-0.6cm} 
    \subfloat[LSRF($S^*$)]{%
         \includegraphics[height=0.20\linewidth]{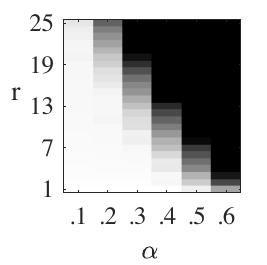}
         \label{fig:lsr_gt_10}}
	\end{tabular}
	\caption{Estimation error $\frac{\|X^*-\hat{X}\|_F} { \|X^*\|_F}$ of UPCA (Algorithm \ref{algo:UPCA}) for sparse permutations ($\alpha \le 0.6)$ and outlier ratio $90\%$, with $\hat{S}$ computed by DPCP \citep{tsakiris2018dual,lerman2018fast} in Stage-I and $\hat{X}$ computed by $\ell_1$-RR \citep{Slawski-JoS19}, PL \citep{Slawski-JCGS2021} or Algorithm \ref{algo:LSRF} in Stage-II.}
	\label{fig:partially shuffled}
\end{figure*}

\subsubsection{Experiments on Face Images} \label{subsection:faces}

In this section we offer a flavor of how the ideas discussed so far apply in a high-dimensional example with real data. We use the well-known database Extended Yale B \citep{GeBeKr01}, which contains fixed-pose face images of distinct individuals, with $64$ images per individual under different illumination conditions. It is well-established that the images of each individual approximately span a low-dimensional subspace. It turns out that for our purpose the value $r = \dim S^* = 4$ is good enough, and values higher than $r$ do not bring significant improvements. Since each image has size $192 \times 168$, the images of each individual can be approximately seen as $n=64$ points $x_j^*, \, j \in [64]$ of a $4$-dimensional linear subspace $S^*$, embedded in an ambient space of dimension $m = 32256$. In what follows we only deal with the images of a fixed individual. We consider four permutation types corresponding to fully or partially ($\alpha=0.4$) permuting image patches of size $16\times 24$ or $48\times 42$, as shown in the second column of Figure \ref{fig:real-face}. To generate a fixed number of $n_{out} = 16$ outliers only one out of the four permutation types is used for each trial. The original images (inliers) together with the ones that have undergone patch-permutation (outliers) are given without any inlier/outlier labels, and the task is to restore all corrupted images. This is a special case of visual permutation learning, recently considered using deep networks \citep{Cruz-CVPR2017,Cruz-PAMI2019}. 

We compute $\hat{S}$ as follows. With $\tilde{X} = U \Sigma V^\top$ the thin SVD of $\tilde{X}$, where $U\in \mathbb{R}^{32256 \times 64}$, DPCP fits a $4$-dimensional subspace $\bar{S}$ to the columns of $\bar{X} = U^\top \tilde{X}$, a process which takes about a tenth of a second. Then $\bar{S}$ is embedded back into $\RR^{32256}$ via the map $U: \RR^{64} \rightarrow \RR^{32256}$ to yield $\hat{S}$. To compute $\hat{X}$ from $\hat{S}$ and $\tilde{X}$ we use the custom algebraic solver of AIEM as well as $\ell_1$-RR, PL, LSRF, with a proximal subgradient implementation of $\ell_1$-RR using the toolbox of \cite{Beck-OMS2019}. 

\renewcommand{\thesubfigure}{\arabic{subfigure}}
\begin{figure}[!htbp]
	\centering  
	\begin{tabular}{c c c c c c}
		{original}    &\hspace{-0.3cm}  \hspace{-0.3cm} {outlier}    &\hspace{-0.30cm} \raisebox{\dimexpr 0cm}{\shortstack{AIEM 
		}}    &\hspace{-0.30cm} \raisebox{\dimexpr 0cm}{\shortstack{$\ell_1$-RR
		}}   &\hspace{-0.30cm} \raisebox{\dimexpr 0cm}{\shortstack{PL 
		}}   &\hspace{-0.3cm} \raisebox{\dimexpr 0cm}{\shortstack{LSRF
		}} \\  \vspace{-0.6cm}
		\subfloat{%
			\includegraphics[height=0.5in]{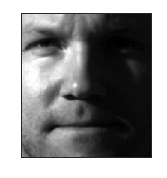}
		}    &\hspace{-0.30cm}
		\subfloat{%
			\includegraphics[height=0.5in]{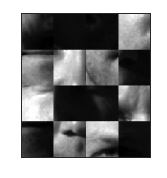}
		}    &\hspace{-0.30cm}
		\subfloat{%
			\includegraphics[height=0.5in]{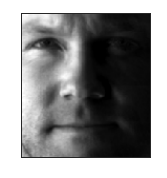}
		}     &\hspace{-0.3cm}  
		&\hspace{-0.30cm}
		&\hspace{-0.3cm}
		\\  \vspace{-0.6cm} 
		{ \ } &\hspace{-0.30cm} 
		\subfloat{%
			\includegraphics[height=0.5in]{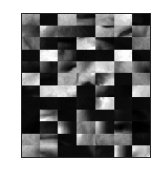}
		}    &\hspace{-0.30cm}
		\subfloat{%
			\includegraphics[height=0.5in]{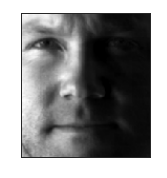}
		}     
		&\hspace{-0.3cm}
		&\hspace{-0.30cm}
		&\hspace{-0.3cm}
		\\  \vspace{-0.6cm} 
		{ \ } &\hspace{-0.30cm} 
		\subfloat{%
			\includegraphics[height=0.5in]{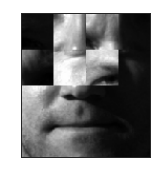}
		}    &\hspace{-0.30cm}
		\subfloat{%
			\includegraphics[height=0.5in]{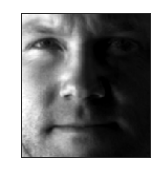}
		}     &\hspace{-0.30cm}
		\subfloat{%
			\includegraphics[height=0.5in]{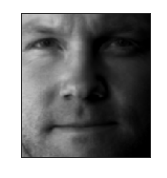}
		}    &\hspace{-0.30cm}
		\subfloat{%
			\includegraphics[height=0.5in]{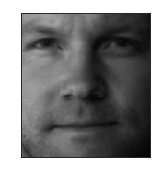}
		}    &\hspace{-0.3cm}
		\subfloat{%
			\includegraphics[height=0.5in]{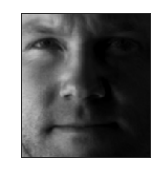}
		} \\ 
		{ \ } &\hspace{-0.30cm} 
		\subfloat{%
			\includegraphics[height=0.5in]{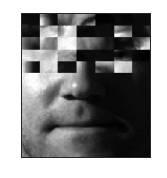}
		}    &\hspace{-0.30cm}
		\subfloat{%
			\includegraphics[height=0.5in]{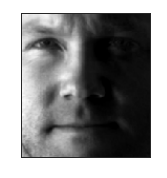}
		}     &\hspace{-0.30cm}
		\subfloat{%
			\includegraphics[height=0.5in]{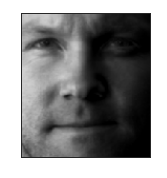}
		}    &\hspace{-0.30cm}
		\subfloat{%
			\includegraphics[height=0.5in]{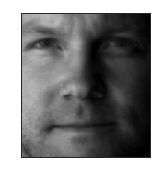}
		}    &\hspace{-0.3cm}
		\subfloat{%
			\includegraphics[height=0.5in]{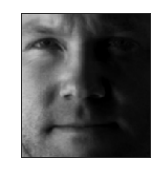}
		} 
	\end{tabular}
	\caption{UPCA on the face dataset Extended Yale B.}
	\label{fig:real-face}
\end{figure}

The first column of Figure \ref{fig:real-face} shows an original image, and the second column shows the corresponding outlier obtained by applying a sample permutation for each of the four different permutation types. Columns three to six give the corresponding point in the output of Algorithm \ref{algo:UPCA} for different unlabeled sensing methods and $\hat{S}$ computed by DPCP \citep{tsakiris2018dual}. CCV-Min \citep{Peng-SPL2020} is not included as branch-and-bound becomes prohibitively expensive for such large $m$). Notably, AIEM \citep{Tsakiris-TIT2020} rather satisfactorily restores the original image regardless of permutation type. 
The performance of the other three methods is shown only for their operational regime, where the given data are corrupted by sparse permutations, and Algorithm \ref{algo:LSRF} most accurately captures the illumination of the original image. Overall, we find these results encouraging, especially if one takes into consideration that the methods are very efficient, requiring only $0.2sec$ (AIEM), $7sec$ ($\ell_1$-RR), $0.2sec$ (PL) and $10sec$ (Algorithm \ref{algo:LSRF}), discounting the DPCP step, which costs $0.1sec$, regardless of permutation type. This is in contrast with existing deep network architectures for visual permutation learning, such as \citep{Cruz-PAMI2019}, which are based on branch-and-bound and thus have in principle  an exponential complexity in the number of permuted patches.

\subsubsection{Experiments on Data Re-identification (UPCA)} \label{section:data re-id}

Finally, we evaluate the UPCA Algorithm \ref{algo:UPCA} for the task of re-identification (section \ref{section:Introduction}) using real educational and medical records and simulated permutations for various sparsity levels $\alpha$, thus emulating a privacy protection scenario. Both of the datasets that we use contain no personally identifiable information. DPCP \citep{tsakiris2018dual,lerman2018fast} computes $\hat{S}$ in Stage-I and $\ell_{1}$-RR \citep{Slawski-JoS19}, PL \citep{Slawski-JCGS2021} or Algorithm \ref{algo:LSRF} produce $\hat{X}$ in Stage-II.

\renewcommand{\thesubfigure}{\alph{subfigure}}
\begin{figure}[!h]
	\centering
	\subfloat[high-school scores]{
		\includegraphics[height=0.3\linewidth]{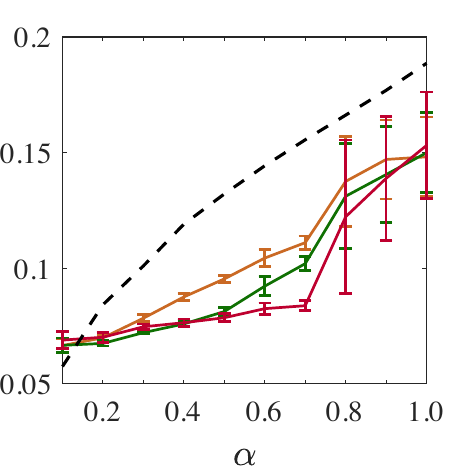}
		\label{fig:score}} 
	\subfloat[breast tumor features]{
		\includegraphics[height=0.3\linewidth]{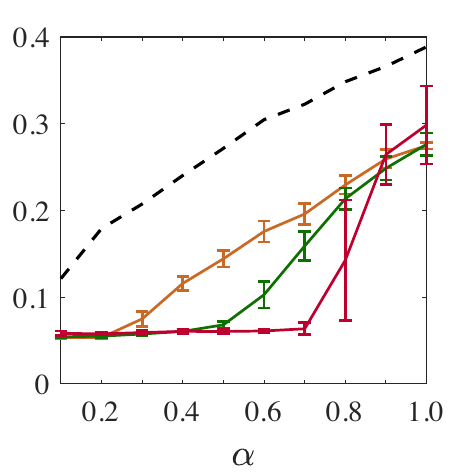}
		\label{fig:breast}} \hspace{-0.4cm} 
	\raisebox{\dimexpr 3cm-\height}{\includegraphics[height=0.11\linewidth]{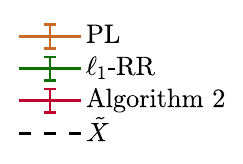}}
	\caption{Relative estimation error $\frac{\|\hat{X}-X^*\|_F}{\|X^*\|_F}$ for UPCA on real data in de-anonymization.
	}\label{fig:recordlink}
\end{figure}

The first dataset consists of the test scores of $m=707$ high-school students on $6$ subjects during two different periods, together with the sum of the score tests for each period, thus $n=14$. For $7$ out of $14$ tests we apply random permutations of the student indices and thus have $50\%$ outliers. With $r=3$, the relative estimation errors on the score records are shown in Figure \ref{fig:score}. The black dashed line depicts the relative difference between the observed data $\tilde{X}$ and the original data $X^*$, which as expected increases for higher $\alpha$'s. The performance of $\ell_{1}$-RR, PL and Algorithm \ref{algo:LSRF} is in alignment with our earlier findings in that Algorithm \ref{algo:LSRF} tends to have a superior performance and PL is the least competitive. All these methods apply in principle for sparse permutations and thus their accuracy naturally degrades for large $\alpha$. 

The second dataset consists of all the benign cases in Breast Cancer Wisconsin (Diagnostic) \citep{Dua:2019}. It has $m=357$ patients and $n=30$ features of a breast mass digitized image for each patient. We randomly permute the patient indices for $15$ of the features thus having $50\%$ outliers and set $r=4$. Figure \ref{fig:breast} shows the relative estimation error of $\hat{X}$ for various permutation sparsity levels $\alpha$, with the unlabeled sensing methods exhibiting the same trend as before. Remarkably, for $\alpha=0.7$, the UPCA Algorithm \ref{algo:UPCA} incorporating Algorithm \ref{algo:LSRF} in Stage-II reduces the original error of the data $\tilde{X}$ from $32.24\%$ to $6.35\%$ in $0.5sec$, as opposed to $15.90\%$ and $19.57\%$ when $\ell_1$-RR \citep{Slawski-JoS19} or PL \citep{Slawski-JCGS2021} are incorporated, respectively.  

\subsection{UMC Experiments}\label{subsection:experiments-UMC}
In this section, we evaluate the proposed two-stage algorithmic pipeline, Algorithm \ref{algo:UMC-proj-then-perm}, for UMC. Section \ref{subsubsection:experiments-UMC-I} tests Stage-I and reports the recovery accuracy of the subspace $S^*$. Section \ref{subsubsection:umc_stage12} presents the performance of the full pipeline in terms of recovering $X^*$.

The experiments operate on synthetic data with the following setup. We set the ambient dimension $m=50$, the overall number of data points $n=100$, the dimension of the ground-truth subspace $r=3$, the noise level is $0.01$, inliers are associated with the dominant identity permutation, and outliers are shuffled by random dense permutations ($\alpha=1$).

\subsubsection{Stage-I of UMC}\label{subsubsection:experiments-UMC-I} 
As discussed in Section \ref{section:UMC pipeline}, Stage-I amounts to solving the problem of matrix completion with column outliers, and for this, we can use two approaches. One is MCO \citep{chen2015matrix}, which simultaneously detects the inliers and estimates the ground-truth subspace $S^*$. The other approach, called DPCP+IST, is an instantiation of our idea in Section \ref{section:UMC pipeline}: (1) detect the inliers applying the DPCP method \citep{tsakiris2018dual} on zero-filled data; (2) complete the detected inliers using a non-convex method based on \textit{iterative soft thresholding} \citep{majumdar2011some}, which we call IST; (3) estimate $S^*$ using an SVD on the matrix of completed inliers. 

With the output $\hat{S}$ of either of the above two approaches, we report the largest principal angle $\theta_{\text{max}}(S^*,\hat{S})$ in Figure  \ref{fig:umc_stage1} for different ratios of outliers ($0.1:0.1:0.9$) and missing entries ($0.1:0.1:0.9$). In particular, via Figure \ref{fig:umc_stage1} we deliver two messages:
\begin{itemize}
    \item Both MCO and DPCP+IST find an accurate enough subspace estimate $\hat{S}$ given sufficiently many inliers and observed entries. For example, we have $\theta_{\text{max}}(S^*,\hat{S})$ equal to $2.44^{\circ}$ for MCO and $5.51^{\circ}$ for DPCP+IST for $10\%$ missing entries and $50\%$ outliers. 
    \item The accuracy of MCO decays more rapidly than DPCP+IST in the presence of more outliers and more missing entries. For example,  with $50\%$ outliers and $50\%$ missing entries, we have $\theta_{\text{max}}(S^*,\hat{S})$ equal to $20.37^{\circ}$ for MCO and $2.61^{\circ}$ for DPCP+IST. Moreover, the figure shows DPCP+IST can handle up to $60\%$ missing entries \& $60\%$ outliers, or $10\%$ missing entries \& $40\%$ outliers, with errors of roughly $5^{\circ}$ when $m=50, n=100,$ and $r=3$. 
\end{itemize}
\begin{figure}[!h]
	\centering
	\raisebox{\dimexpr 2.6cm-\height}{\small{ratio of missing entries}}
	\subfloat[MCO]{
		\includegraphics[height=0.28\linewidth]{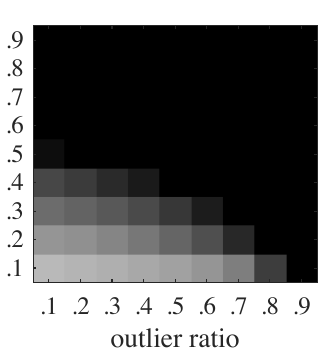}
		\label{fig:mco_stage1}} 
	\hspace{0cm} 
	\subfloat[\scriptsize{DPCP+IST}]{
		\includegraphics[height=0.28\linewidth]{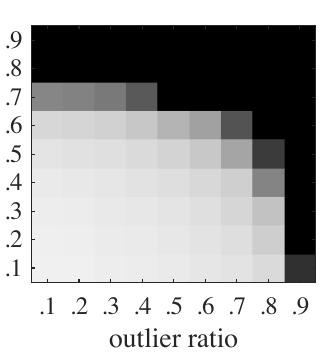}
		\label{fig:dpcpmc_stage1}} 
    \hspace{0.5cm} 
	\includegraphics[height=0.28\linewidth]{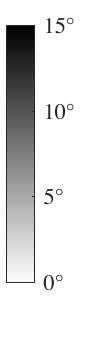}
	\caption{The largest principal angle $\theta_{\text{max}}(S^*,\hat{S})$ for Stage-I of Algorithm \ref{algo:UMC-proj-then-perm} on synthetic data with varying ratios of outliers and missing entries. 
	}\label{fig:umc_stage1}
\end{figure}

Overall, at least in the present setting, DPCP+IST appears to be more accurate and more robust than MCO for Stage-I. This is perhaps because DPCP+IST benefits from decoupling the estimation task into several steps, where each step sufficiently leverages the non-convex structure of the problem. 

Figure \ref{fig:umc_stage1_ranks} reports the largest principal angle $\theta_{\text{max}}(S^*,\hat{S})$ for $r=1:1:20$ when the outlier ratio is $40\%$ and the ratio of missing entries is $20\%$ with $m=50, n=100$. In this setting, the performance starts dropping significantly when the rank is greater than $12$. 
\begin{figure}[!h]
	\centering
	\raisebox{\dimexpr 2.6cm-\height}{\small{$\theta_{\text{max}}(S^*,\hat{S})$}}
	\includegraphics[height=0.25\linewidth]{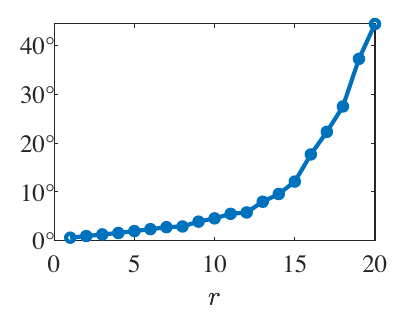}
	\caption{The largest principal angle $\theta_{\text{max}}(S^*,\hat{S})$ for Stage-I of DPCP+IST on synthetic data with varying ranks. 
	}\label{fig:umc_stage1_ranks}
\end{figure}

\subsubsection{The Full Pipeline of UMC} \label{subsubsection:umc_stage12}
We now evaluate the whole two-stage pipeline (Algorithm \ref{algo:UMC-proj-then-perm}) for UMC. We solve Stage-I via DPCP+IST; see Figure \ref{fig:umc_stage1} and Section \ref{subsubsection:experiments-UMC-I}. We solve Stage-II via either AIEM or CCV-Min; see Section \ref{subsection:us-review}.  

In the experiments, we fix the outlier ratio to $40\%$ with varying ratios of missing entries $0.1:0.1:0.8$. For better illustration, we report the recovery accuracy for inliers and outliers separately. In particular, with  the ground-truth matrix $X_{in}^*$ of inliers and the estimated inlier matrix  $\hat{X}_{in}$ given by the algorithm, we report the relative error $\frac{||\hat{X}_{in} - X^*_{in}||_{F}}{||X^*_{in}||_{F}}$ that reflects the recovery accuracy of inliers. The metric $\frac{||\hat{X}_{out} - X^*_{out}||_{F}}{||X^*_{out}||_{F}}$ is defined similarly  for outliers.

\begin{figure}[!h]
	\centering
    \raisebox{\dimexpr 2.9cm-\height}{\footnotesize{relative error}}
	\includegraphics[height=0.33\linewidth]{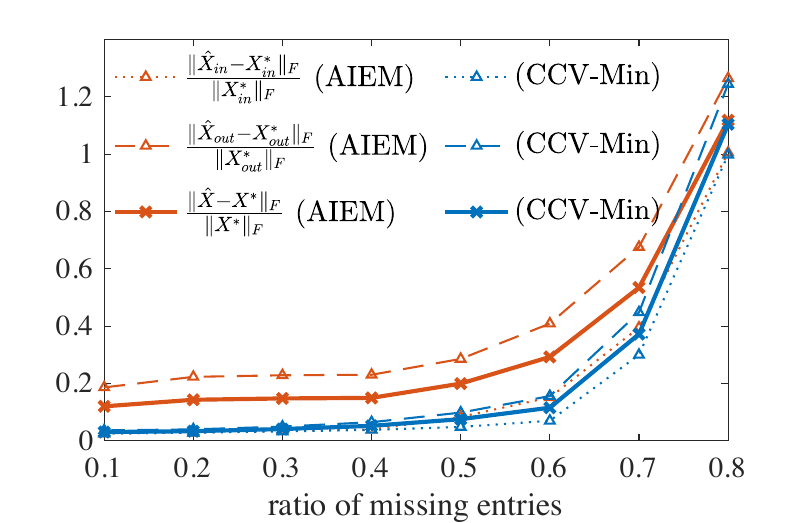}
	\caption{Relative errors for UMC with $m=50$, $n=100$, $r=3$, and outlier ratio $40\%$.}
	\label{fig:umc_stage12_projperm}
\end{figure}

\subsubsection{Experiments on Data Re-identification (UMC)} \label{section:data re-id umc}
Extending the UPCA experiments of Section \ref{section:data re-id}, we now evaluate our UMC algorithm on the medical and educational data. We fix the outlier ratio to be $50\%$, and vary the ratio of missing entries among $0.05:0.05:0.50$. 

\begin{figure}[!h]
	\centering
	\raisebox{\dimexpr 2.5cm-\height}{\footnotesize{relative error}}
	\subfloat[high-school scores]{
		\includegraphics[height=0.28\linewidth]{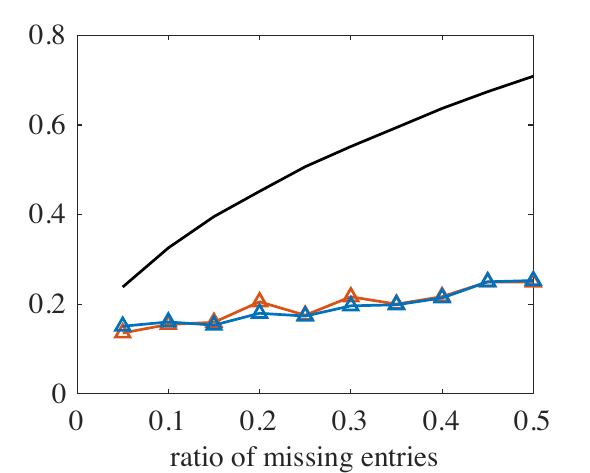}
		\label{fig:score_umc_projperm}} 
	\subfloat[breast tumor features]{
		\includegraphics[height=0.28\linewidth]{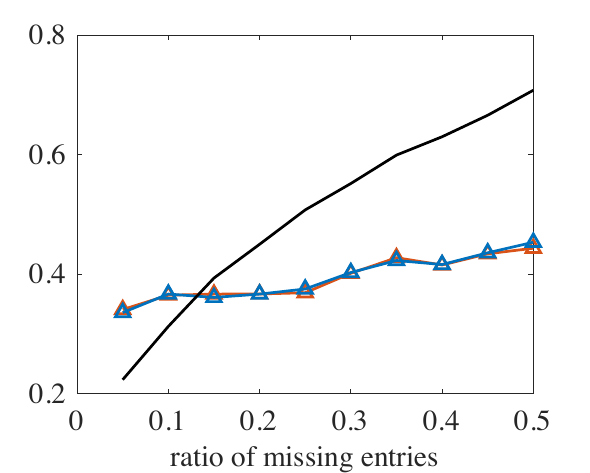}
		\label{fig:breast_umc_projperm}} \hspace{-0.4cm} 
	\raisebox{\dimexpr 3cm-\height}{\includegraphics[height=0.08\linewidth]{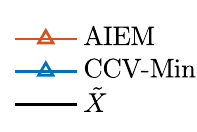}}
	\caption{Relative estimation errors of the proposed UMC pipeline on real data.
	}\label{fig:reid_umc_projperm}
\end{figure}

In Figure \ref{fig:reid_umc_projperm} we present the results. The black curves indicate the distances from the (zero-filled) observed data matrix $\tilde{X}$ to the ground-truth $X^*_{\rm_{in}}$; the distances grow as the ratio of missing entries increases. The proposed UMC algorithm operates on $\tilde{X}$ to restore $X^*$, which is intuitively why it gives smaller relative errors as shown by the colored curves. In particular, for  $25\%$ missing entries, our algorithm reduces the relative error from $51\%$ (black) to $17\%$ (red and blue) for the high-school scores (Figure \ref{fig:score_umc_projperm}) and from $51\%$ to $37\%$ for the breast tumor features (Figure \ref{fig:breast_umc_projperm}).

\section{Proofs}\label{section:proofs}

The proofs of Theorem \ref{thm:main}-\ref{thm:BC-UMC} use basic algebraic geometry and we recall the required notions as we go along. An accessible introduction on the subject is \cite{cox2013ideals}, while a more advanced is \cite{harris2013algebraic}. We also refer to the works \cite{tsakiris2023low, Tsakiris-TIT2020} on matrix completion and linear regression without correspondences, whose mathematical analysis is very related. The proof of Theorem \ref{thm:BC-UMC} is very similar to the proof of Theorem \ref{thm:BC} and is omitted. 

\subsection{Proof of Theorem \ref{thm:main}}
We first prove the theorem over $\CC$, then we transfer the statement over $\RR$. We note here that there is nothing special about $\RR$ and $\CC$ with regards to the problem. Indeed, the same proof applies if one replaces $\RR$ with any infinite field $\mathbb{F}$ and $\CC$ with the algebraic closure $\bar{\mathbb{F}}$ of $\mathbb{F}$. Set 
$$\M_{\CC}= \left. \right\{X\in {\CC}^{m\times n}| \rank_{\CC} X \le r\left. \right\}$$
and note that since $\M_{\CC}$ is irreducible, the intersection of finitely many non-empty open sets in $\M_{\CC}$ is itself non-empty and open, and thus dense. Here irreducibility means that $\M_{\CC}$ can not be decomposed as the union of two proper subvarieties of $\M_{\CC}$. 

\begin{lemma} \label{lem_phi_rank}
	There is an open dense set $\U_1$ in $\M_{\CC}$ such that for any $X\in \U_1$ and any $\underline{\pi} = (\Pi_1,\dots,\Pi_n) \in \prod_{i \in [n]} \P_m$, every $m \times r$ submatrix of $\underline{\pi}(X)$ has rank $r$. 
\end{lemma}
\begin{proof}
	First, fix some $\underline{\pi} = (\Pi_1,\dots,\Pi_n) \in \prod_{i \in [n]} \P_m$ and then some index set $\J = \{j_1, \dots, j_r\} \subset [n]$. 
	The submatrix $\underline{\pi}(X)_\J := \left[\Pi_{j_1} x_{j_1}, \cdots, \Pi_{j_r}x_{j_r}\right]$ of $\underline{\pi}(X)$ has rank less than $r$ if and only if all of its $r \times r$ minors are zero. For each subset $\I=\{i_1,\dots,i_r\} \subset [m]$ we have a polynomial $\det \underline{\pi}(Z)_{\I,\J} \in \CC[Z]$ where $\underline{\pi}(Z)_{\I,\J}$ is the row-submatrix of $\underline{\pi}(Z)_\J$ obtained by selecting the rows with index in $\I$. 
	The set of matrices in ${\CC}^{m \times n}$ for which the evaluation of this polynomial is non-zero is an open set, call it $\U_{\underline{\pi}, \I, \J}$. 
	Then $\underline{\pi}(X)_\J$ has rank $r$ if and only if $X \in \U_{\underline{\pi}, \J}:=\bigcup_{\I} \U_{\underline{\pi}, \I, \J}$, where $\I$ ranges over all subsets of $[m]$ of cardinality $r$. As a union of finitely many open sets, $\U_{\underline{\pi}, \J}$ is open.
	Moreover, every $m \times r$ submatrix of $\underline{\pi}(X)$ has rank $r$ if and only if $X \in \U_{\underline{\pi}}:=\bigcap_{\J} \U_{\underline{\pi}, \J}$, where now $\J$ ranges over all subsets of $[n]$ of cardinality $r$. $\U_{\underline{\pi}}$ is open because it is the finite intersection of open sets. Finally, every $m \times r$ submatrix of $\underline{\pi}(X)$ has rank $r$ for any $\underline{\pi}$ if and only if $X$ is in the open set $\U_1:= \bigcap_{\underline{\pi}} \U_{\underline{\pi}}$, where the intersection is taken over all $\underline{\pi}$'s. 
	
	The proof will be complete once we show that $\U_1$ is non-empty. By what we said above about intersections of finitely many non-empty open sets in an irreducible variety, it is enough to show that each $\U_{\underline{\pi}, \J}$ is non-empty. We do this by constructing a specific $X \in \U_{\underline{\pi}, \J}$. Recall here that any $\Pi \in \P_m$ is diagonalizable over $\CC$ with non-zero eigenvalues. It is an elementary fact in linear algebra that there exists a choice of eigenvector $v_{k}$ of $\Pi_{j_k}$ for every $k \in [r]$ such that $v_{1},\dots,v_{r}$ are linearly independent. Now our $X$ is taken to be the matrix with $v_{k}$ at column $j_k$ for every $k \in [r]$ and zero everywhere else. Clearly $X \in \M_{\CC}$ and moreover $\underline{\pi}(X)_\J=\left[\Pi_{j_1} x_{j_1} \cdots, \Pi_{j_r}x_{j_r}\right] = \left[\Pi_{j_1} v_{1} \cdots, \Pi_{j_r}v_{r}\right] = [\lambda_1 v_1 \cdots \lambda_r v_r]$, where $\lambda_k$ is the corresponding eigenvalue of $v_k$. Since none of the $\lambda_{k}$'s is zero, this matrix has rank $r$, that is $X \in \U_{\underline{\pi}, \J}$. 
\end{proof}

Denote by $\C(X)$ the column-space of $X$ and $I_m$ the identity matrix of size $m \times m$. Note also that whenever $p$ is a non-zero polynomial in $\nu$ variables with coefficients in $\CC$, there is always some $\xi \in {\CC}^\nu$ such that $p(\xi) \neq 0$. 

\begin{lemma} \label{lem_Pi_ne_CX}
	There is an open dense set $\U_2$ in $\M_{\CC}$ such that for any $X \in \U_2$, we have that $\Pi x_{j} \notin \C(X)$ for any $\Pi \in \P_m \setminus\{I_m\}$ and any $j \in [n]$.	 
\end{lemma}
\begin{proof}
	$\Pi x_{j} \notin \C(X)$ if and only if $\rank[X \, \, \Pi x_{j}] = r+1$. As in the proof of Lemma \ref{lem_phi_rank}, this condition is met on an open set $\U_{\Pi,j}$ of $\M_{\CC}$ where some $(r+1) \times (r+1)$ determinant of $[X \, \, \Pi x_{j}]$ is non-zero. 
	Then the statement of the theorem is true on the open set $\U_2 = \bigcap_{\Pi \in \P_m, \, j \in [n]} \U_{\Pi,j}$. As in the proof of Lemma \ref{lem_phi_rank}, to show that $\U_2$ is non-empty it suffices to show that each $\U_{\Pi,j}$ is non-empty. 
	We show the existence of an $X \in \U_{\Pi,j}$. Let $Z=(z_{ik})$ be an $m \times r$ matrix of variables over $\CC$ and consider the polynomial ring $\CC[Z]$. Let us write $z_k$ for the $k$th column of $Z$. Since $\Pi$ is not the identity, there exists some $i \in [m]$ such that $z_{i1}$ is different from the $i$th element of $\Pi z_1$, where $z_1$ is the first column of $Z$. Instead, suppose that the variable $z_{i1}$ appears in the $i'$th coordinate of $\Pi z_1$ with $i' \neq i$. 
	Now take any $\I \subset [m]$ with cardinality $r+1$ such that $i,i' \in \I$ and consider $\det [Z \, \, \Pi z_1 ]_{\I}$ where $[Z \, \, \Pi z_1]_{\I}$ is the submatrix of $[Z \, \, \Pi z_1]$ obtained by selecting the rows with index in $\I$. This is a polynomial of $\CC[Z]$ that has the form $\pm z_{i1}^2 \det [z_2 \cdots z_r]_{\I \setminus \{i,i'\}} + \cdots$ where the remaining terms do not involve $z_{i1}^\nu$ for $\nu >1$. Since the entries of $Z$ are algebraically independent, $\det [z_2 \cdots z_r]_{\I \setminus \{i,i'\}}$ is a non-zero polynomial. We conclude that $\det [Z \, \, \Pi z_1 ]_{\I}$ is also a non-zero polynomial. Hence there exists some $Z' \in {\CC}^{m \times r}$ such that $\det [Z' \, \, \Pi z_1' ]_{\I} \neq 0$. Now define $X$ by setting $x_j = z_1'$,  $x_{j_k} = z_k', \, k \in [r]$ for any choice of $j_k$'s distinct from $j$, and zeros everywhere else. By construction $X \in \U_{\Pi,j}$.   
\end{proof}

Let $f: {\CC}^{m\times r} \times {\CC}^{r\times n} \to \M_{\CC}$ be the surjective map given by $f(B', C') = B'C'$.

\begin{lemma} \label{lem_x_ne_PhiX}
	There is an open dense set $\U_3$ in $\M_{\CC}$ such that for any $X\in \U_3$, we have that for any $j \in [n]$, any $\J=\{j_1,\dots,j_r\} \subset [n]$ with $j \notin \J$ and any $\Pi_1,\dots, \Pi_r \in \P_m$ not all identities, it holds that $\rank [x_j \,  \, \Pi_1 x_{j_1} \cdots \,  \Pi_{r} x_{j_r}] = r+1$. 
\end{lemma}
\begin{proof}
	With $j,\J$ and $\Pi_k$'s fixed, the set $\U_{j,\J,\Pi_1,\dots,\Pi_r}$ of $X$'s in $\M_{\CC}$ for which the rank of $[x_j \, \, \Pi_1 x_{j_1} \cdots  \Pi_r x_{j_r}]$ is $r+1$, is open. Indeed, this is defined by the non-simultaneous vanishing of all $(r+1) \times (r+1)$ minors of $[z_j \, \, \Pi_1 z_{j_1} \cdots  \Pi_r z_{j_r}]$, where $z_k$ is the $k$th column of the matrix of variables $Z$ from the proof of Lemma \ref{lem_Pi_ne_CX}. We note that these are polynomials in $Z$ with integer coefficients. Set $\U_3 = \bigcap_{j, \J, \Pi_1,\dots,\Pi_r} \U_{j,\J,\Pi_1,\dots,\Pi_r}$ where the intersection is taken over all choices of $j, \J, \Pi_1,\dots,\Pi_r$ as in the statement of the lemma. As in the proof of Lemma \ref{lem_phi_rank}, the set $\U_3$ is open and to show that it is non-empty is suffices to show that each $\U_{j,\J,\Pi_1,\dots,\Pi_r}$ is non-empty. 
	
	Let $\U_1, \U_2$ be the open sets of Lemmas \ref{lem_phi_rank} and \ref{lem_Pi_ne_CX}. 
	Since $\M_{\CC}$ is irreducible and $\U_1, \U_2$ are open and non-empty, we have that $\U_1 \cap \U_2$ is non-empty. 
	Since $f$ is surjective, $f^{-1}(\U_1 \cap \U_2)$ is also non-empty. Take any $(B',C') \in f^{-1}(\U_1 \cap \U_2)$. By definition, the rank of $[\Pi_1 B'c'_{j_1} \cdots  \Pi_r B'c'_{j_r}]$ is $r$. By hypothesis, there is some $k \in [r]$ such that $\Pi_{k}$ is not the identity and thus again by definition we have $\rank[B' \, \, \Pi_{k}B' c_{j_{k}}'] = r+1$. Consequently, $\Pi_{k}B' c_{j_{k}}' \notin \C(B')$ and so the two $r$-dimensional subspaces $\C(B')$ and $\C\big([\Pi_1 B'c'_{j_1} \cdots  \Pi_r B'c'_{j_r}]\big)$ are distinct. Thus there exists some $c'' \in {\CC}^r$ such that $B'c'' \not\in \C\big([\Pi_1 B'c'_{j_1} \cdots  \Pi_r B'c'_{j_r}]\big)$. Define $C'' \in {\CC}^{r \times n}$ by setting $c_\nu'' = c_\nu'$ for every $\nu \neq j$ and $c_j'' = c''$. Then by construction $B' C'' \in \U_{j,\J,\Pi_1,\dots,\Pi_r}$.
\end{proof}

Take $X^*=[x_1^*\cdots x_n^*] \in \U_3$ and let $\tilde{X} = [\tilde{\Pi}_1 x_1^* \cdots \tilde{\Pi}_n x_n^*]$. Now 
$\rank \tilde{X} = \rank \tilde{\Pi}_1^{-1}  \tilde{X} = \rank [x_1^* \, \, \tilde{\Pi}_1^{-1}\tilde{\Pi}_2 x_2^* \cdots \tilde{\Pi}_1^{-1}\tilde{\Pi}_n x_n^*]$. If there is some $k \ge 2$ such that $\tilde{\Pi}_1 \neq \tilde{\Pi}_k$, by Lemma \ref{lem_x_ne_PhiX} any $m \times (r+1)$ submatrix of $\tilde{\Pi}_1^{-1}  \tilde{X}$ that contains columns $1$ and $k$ will have rank $r+1$. On the other hand, when all $\tilde{\Pi}_k$'s are equal for $k \in [n]$, the rank of $\tilde{X}$ is $r$ by Lemma \ref{lem_phi_rank}. This concludes the proof of the theorem over $\CC$ with the claimed open set being $\U_3$, which we denote in the sequel by $\U_\CC$. 

Set $\M_\RR= \left. \right\{X\in \RR^{m\times n}| \rank_{\RR} X \le r\left. \right\}$. There is an inclusion of sets $i: \M_\RR \hookrightarrow \M_{\CC}$ where for $X \in \M_\RR$ we view $i(X)$ as the complex matrix associated to $X$. The reason for this inclusion is that if the columns of $X$ generate an $r$-dimensional subspace over $\RR$, then they generate an $r$-dimensional subspace over $\CC$. To finish the proof, it suffices to show the existence of a non-empty open set $\U_\RR$ in $\M_\RR$ such that $i(\U_\RR) \subset \U_\CC$. This comes from two key ingredients. 
The first one is the observation that the polynomials that induce $\U_\CC$, i.e. the polynomials of $\CC[Z]$ whose non-simultaneous vanishing indicates membership of a point $X \in \M_{\CC}$ in $\U_\CC$, they have integer and thus real coefficients. This can be seen by inspecting the proof of Lemma \ref{lem_x_ne_PhiX}. 
Call the set of these polynomials $\mathfrak{p}_\U \subset \mathbb{Z}[Z]$. For the second ingredient, let $\mathfrak{p}_\M \subset \mathbb{Z}[Z]$ be the set of all $(r+1) \times (r+1)$ minors of the matrix of variables $Z$. It is a matter of linear algebra that $\M_\RR$ and $\M_{\CC}$ are the common roots of the polynomial system $\mathfrak{p}_\M$ over $\RR^{m \times n}$ and ${\CC}^{m \times n}$ respectively. What is instead a difficult theorem in commutative algebra is that the following algebraic converse is true; see Section 2.6 in \cite{tsakiris2023low}: a polynomial $q \in \RR[Z]$ vanishes on every point of $\M_\RR$ if and only if it is a polynomial combination of elements of $\mathfrak{p}_\M$, that is if and only if $q = \sum_{p \in \mathfrak{p}_\M} c_p \, p$ for some $c_p$'s in $\RR[Z]$. This statement also holds true if we replace $\RR$ with $\CC$. Now the set $\U_\CC$ consists of those points of $\M_{\CC}$ that are roots of the polynomial system $\mathfrak{p}_\M$ but not of $\mathfrak{p}_\U$. Since $\U_\CC$ is non-empty, not all polynomials in $\mathfrak{p}_\U$ are polynomial combinations of $\mathfrak{p}_\M$. But then, by what we just said, not all points of $\M_\RR$ are common roots of $\mathfrak{p}_\U$. This means that the open set of $\M_\RR$ defined by the non-simultaneous vanishing of all polynonials in $\mathfrak{p}_\U$ is non-empty. This open set is the claimed $\U$.

\subsection{Proof of Theorem \ref{thm:BC} and Theorem \ref{thm:BC-UMC}}

Let $\U_1$ be the open set of Theorem 1. Let $\U_2$ be the set of $X$'s for which $\C(X)$ does not drop dimension under projection onto any $r$ coordinates. This set is open in $\M$ because $X \in \U_2$ if and only if for any $\I \subset [m]$ of cardinality $r$ not all $r \times r$ minors of $X_\I$ are zero, $X_\I$ being the row-submatrix of $X$ obtained by selecting the rows with index in $\I$. Set $\U = \U_1 \cap \U_2$. Then for any $X^* \in \U$ and any $\Pi \in \P_m$ there is a unique factorization $\Pi X^* = B_\Pi^* C_\Pi^*$ with the top $r \times r$ block of $B_\Pi^* \in \RR^{m \times r}$ being the identity. Since $\bar{p}_{\ell,j}(\tilde{X}) =\bar{p}_{\ell,j}(X^*)= \bar{p}_{\ell,j}(\Pi X^*) = \bar{p}_{\ell,j}(B_\Pi^* C_\Pi^*)$ we have that $(B_\Pi^*, C_\Pi^*) \in \Y_{X^*}$ for every $\Pi \in \P_m$. For the reverse direction we recall a basic fact (see Lemma 2 of \cite{Song-ISIT18}): 
\begin{lemma}\label{lem:sym}
	Fix any $j \in [n]$. Suppose that $\xi_1,\xi_2 \in \RR^m$ are such that $\bar{p}_{\ell,j}(\xi_1) = \bar{p}_{\ell,j}(\xi_2)$ for every $\ell \in [m]$. Then $\xi_1 = \Pi \xi_2$ for some $\Pi \in \P_m$. 
\end{lemma}
Now let $(B',C') \in \Y_{X^*}$ and write $c_j'$ for the $j$th column of $C'$. For a fixed $j \in [n]$ the equations $q_{\ell,j}(B',C')=0$ are equivalent to $\bar{p}_{\ell,j}(B' c_j') = \bar{p}_{\ell,j}(x^*_j)$ for every $\ell \in [m]$. By Lemma \ref{lem:sym} there must exist some $\Pi_j \in \P_m$ such that $B' c_j' = \Pi_j x_j^*$. This is true for every $j \in [n]$ so that $B' C' = [\Pi_1 x_1^* \cdots \Pi_n x_n^*]$. This implies that $\rank [\Pi_1 x_1^* \cdots \Pi_n x_n^*] = r$. Since $X^* \in \U$, Theorem 1 gives that all $\Pi_j$'s must be the same permutation $\Pi \in \P_m$, so that $B' C' = \Pi X^*$. Since by construction for any $(B'',C'') \in \Y_{X^*}$ the top $r \times r$ block of $B''$ is the identity, we have that $B'= B_\Pi^*$ and thus necessarily $C'= C_\Pi^*$. 

The proof of Theorem \ref{thm:BC-UMC} is very similar to the proof of Theorem \ref{thm:BC} and is omitted.

\subsection{Proof of Theorem \ref{thm:dominant}}

We first notice $\# \{ \tilde{x}_j \, | \, \tilde{x}_j \in S^* \, ; \, j \in [n]\} \ge \mu(I_m)\ge r+1$. 
Now we suppose $\tilde{x}_{j_1}, \dots,\tilde{x}_{j_r}, \tilde{x}_{j_{r+1}}$ are $r+1$ points in $\tilde{X}$ such that not all $\Pi_{j_1}, \dots, \Pi_{j_r}, \Pi_{j_{r+1}}$ are the identity $I_m$. 
Since $\mu(\Pi)<r$ for $\Pi\ne I_m$, it is impossible that $\Pi_{j_1}= \dots=\Pi_{j_r}= \Pi_{j_{r+1}}$. According to Theorem 1, the points $\tilde{x}_{j_1}, \dots,\tilde{x}_{j_r}, \tilde{x}_{j_{r+1}}$ span a subspace of dimension $r+1$.  
Hence, for any subspace $S\ne S^*$ with $\dim(S)\le r$, we have $\# \{ \tilde{x}_j \, | \, \tilde{x}_j \in S \, ; \, j \in [n]\} \le r$. \hfill 

\subsection{Proof of Theorem \ref{thm:main-UMC} }

As before, we first consider the problem over $\CC$. The transfer to $\RR$ follows the same argument as in the proof of Theorem 1 in \cite{tsakiris2023low} and is omitted. We use the same letters $\underline{p}_\Omega: \M_\CC \rightarrow \CC^\Omega$ and $\underline{\pi}_\Omega :  \CC^\Omega \rightarrow  \CC^\Omega$ to indicate the same maps between the corresponding spaces over $\CC$. 

The map $\underline{p}_\Omega$ is defined by $x_{ij} \mapsto x_{ij}$ if $(i,j) \in \Omega$ and $x_{ij} \mapsto 0$ if $(i,j) \in \Omega^c$. These are polynomial functions in the $x_{ij}$'s so that $\underline{p}_\Omega$ is a \emph{morphism} of irreducible algebraic varieties. In particular, $\underline{p}_\Omega$ is continuous in the Zariski topology, and thus inverse images of open sets are open. Now, under the hypothesis on $\Omega$, it was shown in \cite{tsakiris2023low,tsakiris2020results} that $\Omega$ is generically finitely completable. This is equivalent to the existence of a dense open set $U_0$ in $\M_\CC$ such that for every $X^* \in U_0$ the fiber $\underline{p}_\Omega^{-1}(X^*)$ is a finite set. It is also equivalent to saying that $\underline{p}_\Omega$ is dominant, in the sense that the image $\underline{p}_\Omega(\M_\CC)$ of $\underline{p}_\Omega$ is a dense set in $\CC^\Omega$, that is the closure of $\underline{p}_\Omega(\M_\CC)$ is $\CC^\Omega$.

\begin{lemma} \label{lem_Chevalley}
	The image $\underline{p}_\Omega(U_0)$ of $U_0$ under $\underline{p}_\Omega$ contains a non-empty open set of $\CC^\Omega$.
\end{lemma}
\begin{proof}
	A locally closed set is the intersection of a closed set with an open set. A constructible set is the finite union of locally closed sets. Chevalley's theorem says that a morphism of algebraic varieties takes a constructible set to a constructible set. Since $U_0$ is open, it is constructible, and thus $\underline{p}_\Omega(U_0)$ is also constructible. Hence, there exists a positive integer $s$, closed sets $Y_k, k \in [s]$ and non-empty open sets $V_k, k \in [s]$ of $\M_\CC$ such that $\underline{p}_\Omega(U_0) = \bigcup_{k \in [s]} Y_k \cap V_k$. For the sake of contradiction, suppose that $\underline{p}_\Omega(U_0)$ does not contain any non-empty open set. Then necessarily all $Y_k$'s are proper closed sets, otherwise some $V_k$ is contained in $\underline{p}_\Omega(U_0)$. Hence $\underline{p}_\Omega(U_0)$ is contained in the closed set $Y=\bigcup_{k \in [s]} Y_k$. This is a proper closed set because $\M_\CC$ is irreducible. But then the closure of $\underline{p}_\Omega(U_0)$ is contained in $Y$, which contradicts the fact that $\underline{p}_\Omega(U_0)$ is dense in $\CC^\Omega$. 
\end{proof}

By Lemma \ref{lem_Chevalley} $\underline{p}_\Omega(U_0)$ contains a non-empty open set $V_0$. Now each $\underline{\pi}_\Omega \in \P_{\Omega}$ is a bijective polynomial function on $\CC^\Omega$. Hence $\underline{\pi}_\Omega: \CC^\Omega \rightarrow \CC^\Omega$ is a homeomorphism of topological spaces, so that $\underline{\pi}_\Omega(V_0)$ is a non-empty open set of $\CC^\Omega$. Define $V = \bigcap_{\underline{\pi}_\Omega \in \P_{\Omega}} \underline{\pi}_\Omega(V_0)$. It is a non-empty open set of $\CC^\Omega$. 

\begin{lemma} \label{lem_V_invariance}
	For any $\underline{\pi}_\Omega' \in \P_{\Omega}$ we have that $\underline{\pi}_\Omega'(V) = V$. 
\end{lemma}
\begin{proof}
	If $f: S \rightarrow T$ is a one-to-one (injective) function of sets and $S_1, S_2$ are subsets of $S$, then we always have $f(S_1 \cap S_2) = f(S_1) \cap f(S_2)$. Hence $\underline{\pi}_\Omega'(V) = \bigcap_{\underline{\pi}_\Omega \in \P_{\Omega}} \underline{\pi}_\Omega' \circ \underline{\pi}_\Omega(V_0)$. But $\P_{\Omega}$ is a group under composition of functions so that the coset $\underline{\pi}_\Omega' \P_{\Omega} = \{\underline{\pi}_\Omega' \circ \underline{\pi}_\Omega \, | \, \underline{\pi}_\Omega \in \P_{\Omega} \}$ is equal to $\P_{\Omega}$. 
\end{proof}

%.
As noted earlier, $\underline{p}_\Omega$ is continuous and so the set $U = \underline{p}_\Omega^{-1}(V)$ is open. Moreover, it is non-empty since $V$ is a subset of $V_0$ which is a subset of the image of $\underline{p}_\Omega$. Hence $U$ is dense in $\M_\CC$. Suppose that $X^* \in U$ and set $\tilde{X} = \tilde{\underline{\pi}}_\Omega \underline{p}_\Omega(X^*)$ for some $\tilde{\underline{\pi}}_\Omega \in \P_{\Omega}$. It is enough to show that as $\underline{\pi}_\Omega$ ranges in $\P_{\Omega}$ there are only finitely many $X \in \M_\CC$ such that $\underline{\pi}_\Omega \underline{p}_\Omega(X) =  \tilde{X}$. This equation can be written as $\underline{p}_\Omega(X) = \underline{\pi}_\Omega' \underline{p}_\Omega(X^*)$ with $\underline{\pi}_\Omega' = \underline{\pi}_\Omega^{-1} \circ \tilde{\underline{\pi}}_\Omega$. Since $X^* \in U$, $\underline{p}_\Omega(X^*) \in V$. Thus $\underline{\pi}_\Omega' \underline{p}_\Omega(X^*) \in \underline{\pi}'_\Omega(V)$. By Lemma \ref{lem_V_invariance} we have $\underline{\pi}_\Omega'(V) = V$ and so $\underline{\pi}_\Omega' \underline{p}_\Omega(X^*) \in V$. But $V \subset V_0$ so that $\underline{\pi}_\Omega' \underline{p}_\Omega(X^*) \in V_0$. Since $V_0$ is a subset of $\underline{p}_\Omega(U_0)$, there is some $X_0 \in U_0$ such that $\underline{\pi}_\Omega' \underline{p}_\Omega(X^*) = \underline{p}_\Omega(X_0)$. By definition of $U_0$ the fiber $\underline{p}_\Omega^{-1}\big( \underline{p}_\Omega(X_0)\big)$ is a finite set. But $\underline{p}_\Omega^{-1}\big( \underline{p}_\Omega(X_0)\big) = \underline{p}_\Omega^{-1}\big( \underline{\pi}_\Omega' \underline{p}_\Omega(X^*)\big)$ so that there are finitely many $X$'s in $\M$ that map under $\underline{p}_\Omega$ to $\underline{\pi}_\Omega' \underline{p}_\Omega(X^*)$. As there are finitely many choices for $\underline{\pi}_\Omega'$, we are done.

\vskip 0.2in
\bibliographystyle{abbrvnat}
\bibliography{upcabib,Liangzu,UPCA-Manolis,MC-JMLR20}

\begin{thebibliography}{98}
\providecommand{\natexlab}[1]{#1}
\providecommand{\url}[1]{\texttt{#1}}
\expandafter\ifx\csname urlstyle\endcsname\relax
  \providecommand{\doi}[1]{doi: #1}\else
  \providecommand{\doi}{doi: \begingroup \urlstyle{rm}\Url}\fi

\bibitem[{Abid} and {Zou}(2018)]{Abid-Allerton2018}
A.~{Abid} and J.~{Zou}.
\newblock A stochastic expectation-maximization approach to shuffled linear
  regression.
\newblock In \emph{Annual Allerton Conference on Communication, Control, and
  Computing}, 2018.

\bibitem[Abowd(2019)]{abowd2019stepping}
J.~M. Abowd.
\newblock Stepping-up: The census bureau tries to be a good data steward in the
  21st century.
\newblock 2019.

\bibitem[Antoni and Schnell(2019)]{Antoni-2019}
M.~Antoni and R.~Schnell.
\newblock The past, present and future of the german record linkage center.
\newblock \emph{Jahrbücher für Nationalökonomie und Statistik}, 239\penalty0
  (2):\penalty0 319 -- 331, 2019.

\bibitem[Azadkia and Balabdaoui(2022)]{azadkia2022linear}
M.~Azadkia and F.~Balabdaoui.
\newblock Linear regression with unmatched data: a deconvolution perspective.
\newblock \emph{arXiv preprint arXiv:2207.06320}, 2022.

\bibitem[Balcan et~al.(2019)Balcan, Liang, Song, Woodruff, and
  Zhang]{balcan2019non}
M.-F. Balcan, Z.~Liang, Y.~Song, D.~P. Woodruff, and H.~Zhang.
\newblock Non-convex matrix completion and related problems via strong duality.
\newblock \emph{Journal of Machine Learning Research}, 20\penalty0
  (102):\penalty0 1--56, 2019.

\bibitem[Balzano et~al.(2010)Balzano, Nowak, and Recht]{balzano2010online}
L.~Balzano, R.~Nowak, and B.~Recht.
\newblock Online identification and tracking of subspaces from highly
  incomplete information.
\newblock In \emph{Annual Allerton Conference on Communication, Control, and
  Computing}, pages 704--711. IEEE, 2010.

\bibitem[Balzano et~al.(2018)Balzano, Chi, and Lu]{balzano2018streaming}
L.~Balzano, Y.~Chi, and Y.~M. Lu.
\newblock Streaming {PCA} and subspace tracking: The missing data case.
\newblock \emph{Proceedings of the IEEE}, 106\penalty0 (8):\penalty0
  1293--1310, 2018.

\bibitem[Bates et~al.()Bates, Hauenstein, Sommese, and Wampler]{Bertini}
D.~J. Bates, J.~D. Hauenstein, A.~J. Sommese, and C.~W. Wampler.
\newblock Bertini: Software for numerical algebraic geometry.

\bibitem[Beck and Guttmann-Beck(2019)]{Beck-OMS2019}
A.~Beck and N.~Guttmann-Beck.
\newblock {FOM} – a matlab toolbox of first-order methods for solving convex
  optimization problems.
\newblock \emph{Optimization Methods and Software}, 34\penalty0 (1):\penalty0
  172--193, 2019.

\bibitem[Bernstein(2017)]{bernstein2017completion}
D.~Bernstein.
\newblock Completion of tree metrics and rank 2 matrices.
\newblock \emph{Linear Algebra and its Applications}, 533:\penalty0 1--13,
  2017.

\bibitem[Bertsimas and Li(2020)]{bertsimas2020fast}
D.~Bertsimas and M.~L. Li.
\newblock Fast exact matrix completion: A unified optimization framework for
  matrix completion.
\newblock \emph{Journal of Machine Learning Research}, 21\penalty0
  (231):\penalty0 1--43, 2020.

\bibitem[Breiding et~al.(2018)Breiding, Kali{\v{s}}nik, Sturmfels, and
  Weinstein]{breiding2018learning}
P.~Breiding, S.~Kali{\v{s}}nik, B.~Sturmfels, and M.~Weinstein.
\newblock Learning algebraic varieties from samples.
\newblock \emph{Revista Matem{\'{a}}tica Complutense}, 31\penalty0
  (3):\penalty0 545--593, 2018.

\bibitem[Breiding et~al.(2023)Breiding, Gesmundo, Micha{\l}ek, and
  Vannieuwenhoven]{breiding2023algebraic}
P.~Breiding, F.~Gesmundo, M.~Micha{\l}ek, and N.~Vannieuwenhoven.
\newblock Algebraic compressed sensing.
\newblock \emph{Applied and Computational Harmonic Analysis}, 65:\penalty0
  374--406, 2023.

\bibitem[Brown et~al.(1990)Brown, Cocke, Della~Pietra, Della~Pietra, Jelinek,
  Lafferty, Mercer, and Roossin]{Brown-1990}
P.~F. Brown, J.~Cocke, S.~A. Della~Pietra, V.~J. Della~Pietra, F.~Jelinek,
  J.~D. Lafferty, R.~L. Mercer, and P.~S. Roossin.
\newblock A statistical approach to machine translation.
\newblock \emph{Computational Linguistics}, 16\penalty0 (2):\penalty0 79--85,
  1990.

\bibitem[Cai et~al.(2010)Cai, Cand{\`e}s, and Shen]{cai2010singular}
J.-F. Cai, E.~J. Cand{\`e}s, and Z.~Shen.
\newblock A singular value thresholding algorithm for matrix completion.
\newblock \emph{SIAM Journal on optimization}, 20\penalty0 (4):\penalty0
  1956--1982, 2010.

\bibitem[Candes and Plan(2010)]{candes2010matrix}
E.~J. Candes and Y.~Plan.
\newblock Matrix completion with noise.
\newblock \emph{Proceedings of the IEEE}, 98\penalty0 (6):\penalty0 925--936,
  2010.

\bibitem[Cand{\`e}s and Recht(2009)]{candes2009exact}
E.~J. Cand{\`e}s and B.~Recht.
\newblock Exact matrix completion via convex optimization.
\newblock \emph{Foundations of Computational Mathematics}, 9\penalty0
  (6):\penalty0 717--772, 2009.

\bibitem[Cand{\`e}s et~al.(2011)Cand{\`e}s, Li, Ma, and
  Wright]{candes2011robust}
E.~J. Cand{\`e}s, X.~Li, Y.~Ma, and J.~Wright.
\newblock Robust principal component analysis?
\newblock \emph{Journal of the ACM}, 58\penalty0 (3):\penalty0 1--37, 2011.

\bibitem[Chen et~al.(2011)Chen, Xu, Caramanis, and Sanghavi]{chen2011robust}
Y.~Chen, H.~Xu, C.~Caramanis, and S.~Sanghavi.
\newblock Robust matrix completion and corrupted columns.
\newblock In \emph{International Conference on Machine Learning}, pages
  873--880, 2011.

\bibitem[Chen et~al.(2015)Chen, Xu, Caramanis, and Sanghavi]{chen2015matrix}
Y.~Chen, H.~Xu, C.~Caramanis, and S.~Sanghavi.
\newblock Matrix completion with column manipulation: Near-optimal
  sample-robustness-rank tradeoffs.
\newblock \emph{IEEE Transactions on Information Theory}, 62\penalty0
  (1):\penalty0 503--526, 2015.

\bibitem[Cox et~al.(2013)Cox, Little, and OShea]{cox2013ideals}
D.~Cox, J.~Little, and D.~OShea.
\newblock \emph{{Ideals, Varieties, and Algorithms: An Introduction to
  Computational Algebraic Geometry and Commutative Algebra}}.
\newblock Springer Science \& Business Media, 2013.

\bibitem[Davenport and Romberg(2016)]{davenport2016overview}
M.~A. Davenport and J.~Romberg.
\newblock An overview of low-rank matrix recovery from incomplete observations.
\newblock \emph{IEEE Journal of Selected Topics in Signal Processing},
  10\penalty0 (4):\penalty0 608--622, 2016.

\bibitem[Ding et~al.(2021)Ding, Zhu, Vidal, and Robinson]{ding2021dual}
T.~Ding, Z.~Zhu, R.~Vidal, and D.~P. Robinson.
\newblock Dual principal component pursuit for robust subspace learning: Theory
  and algorithms for a holistic approach.
\newblock In \emph{International Conference on Machine Learning}, 2021.

\bibitem[{Dokmanic}(2019)]{Dokmanic-SPL2019}
I.~{Dokmanic}.
\newblock Permutations unlabeled beyond sampling unknown.
\newblock \emph{IEEE Signal Processing Letters}, 26\penalty0 (6):\penalty0
  823--827, 2019.

\bibitem[Domingo-Ferrer and Muralidhar(2016)]{DomingoFerrer-IS2016}
J.~Domingo-Ferrer and K.~Muralidhar.
\newblock New directions in anonymization: Permutation paradigm, verifiability
  by subjects and intruders, transparency to users.
\newblock \emph{Information Sciences}, 337-338:\penalty0 11 -- 24, 2016.

\bibitem[Dua and Graff(2017)]{Dua:2019}
D.~Dua and C.~Graff.
\newblock {UCI} machine learning repository, 2017.

\bibitem[Eftekhari et~al.(2019)Eftekhari, Ongie, Balzano, and
  Wakin]{eftekhari2019streaming}
A.~Eftekhari, G.~Ongie, L.~Balzano, and M.~B. Wakin.
\newblock Streaming principal component analysis from incomplete data.
\newblock \emph{Journal of Machine Learning Research}, 20:\penalty0 86--1,
  2019.

\bibitem[Elhami et~al.(2017)Elhami, Scholefield, Haro, and
  Vetterli]{Elhami-ICASSP17}
G.~Elhami, A.~Scholefield, B.~B. Haro, and M.~Vetterli.
\newblock Unlabeled sensing: Reconstruction algorithm and theoretical
  guarantees.
\newblock In \emph{IEEE International Conference on Acoustics, Speech and
  Signal Processing}, pages 4566--4570, 2017.

\bibitem[Eriksson et~al.(2012)Eriksson, Balzano, and Nowak]{eriksson2012high}
B.~Eriksson, L.~Balzano, and R.~Nowak.
\newblock High-rank matrix completion.
\newblock In \emph{Artificial Intelligence and Statistics}, pages 373--381.
  PMLR, 2012.

\bibitem[Fellegi and Sunter(1969)]{Fellegi-1969}
I.~P. Fellegi and A.~B. Sunter.
\newblock A theory for record linkage.
\newblock \emph{Journal of the American Statistical Association}, 64\penalty0
  (328):\penalty0 1183--1210, 1969.

\bibitem[Ganti et~al.(2015)Ganti, Balzano, and Willett]{ganti2015matrix}
R.~S. Ganti, L.~Balzano, and R.~Willett.
\newblock Matrix completion under monotonic single index models.
\newblock In \emph{Advances in Neural Information Processing Systems}, 2015.

\bibitem[Georghiades et~al.(2001)Georghiades, Belhumeur, and
  Kriegman]{GeBeKr01}
A.~S. Georghiades, P.~N. Belhumeur, and D.~J. Kriegman.
\newblock From few to many: Illumination cone models for face recognition under
  variable lighting and pose.
\newblock \emph{IEEE Transactions on Pattern Analysis and Machine
  Intelligence}, 23\penalty0 (6):\penalty0 643--660, 2001.

\bibitem[Harris(1992)]{harris2013algebraic}
J.~Harris.
\newblock \emph{{Algebraic Geometry: A First Course}}, volume 133.
\newblock Springer Science \& Business Media, 1992.

\bibitem[He et~al.(2011)He, Xiao, Li, Wang, Wang, and Shi]{He2011PermutationAI}
X.~He, Y.~Xiao, Y.~Li, Q.~Wang, W.~Wang, and B.~Shi.
\newblock Permutation anonymization: Improving anatomy for privacy preservation
  in data publication.
\newblock In \emph{PAKDD Workshops}, 2011.

\bibitem[Hsu et~al.(2017)Hsu, Shi, and Sun]{Hsu-NIPS17}
D.~Hsu, K.~Shi, and X.~Sun.
\newblock Linear regression without correspondence.
\newblock In \emph{Advances in Neural Information Processing Systems}, 2017.

\bibitem[Ji et~al.(2014)Ji, Li, Salzmann, and Dai]{ji2014robust}
P.~Ji, H.~Li, M.~Salzmann, and Y.~Dai.
\newblock Robust motion segmentation with unknown correspondences.
\newblock In \emph{European Conference on Computer Vision}, pages 204--219,
  2014.

\bibitem[Keshavan et~al.(2010)Keshavan, Montanari, and Oh]{keshavan2010matrix}
R.~H. Keshavan, A.~Montanari, and S.~Oh.
\newblock Matrix completion from a few entries.
\newblock \emph{IEEE Transactions on Information Theory}, 56\penalty0
  (6):\penalty0 2980--2998, 2010.

\bibitem[Kir{\'a}ly and Tomioka(2012)]{kiraly2012combinatorial}
F.~Kir{\'a}ly and R.~Tomioka.
\newblock A combinatorial algebraic approach for the identifiability of
  low-rank matrix completion.
\newblock In \emph{International Conference on Machine Learning}, pages
  755--762, 2012.

\bibitem[Kir{\'a}ly et~al.(2015)Kir{\'a}ly, Theran, and
  Tomioka]{kiraly2015algebraic}
F.~Kir{\'a}ly, L.~Theran, and R.~Tomioka.
\newblock The algebraic combinatorial approach for low-rank matrix completion.
\newblock \emph{Journal of Machine Learning Research}, 16:\penalty0 1391--1436,
  2015.

\bibitem[Kleiman and Landolfi(1971)]{kleiman1971geometry}
S.~L. Kleiman and J.~Landolfi.
\newblock Geometry and deformation of special schubert varieties.
\newblock \emph{Compositio Mathematica}, 23\penalty0 (4):\penalty0 407--434,
  1971.

\bibitem[Larsson et~al.(2017)Larsson, Astrom, and Oskarsson]{Larsson-CVPR2017}
V.~Larsson, K.~Astrom, and M.~Oskarsson.
\newblock Efficient solvers for minimal problems by syzygy-based reduction.
\newblock In \emph{IEEE Conference on Computer Vision and Pattern Recognition},
  pages 2383--2392, 2017.

\bibitem[Lenz et~al.(2022)Lenz, Siegel, Wachter-Zeh, and Yaakobi]{Lenz-TIT2022}
A.~Lenz, P.~H. Siegel, A.~Wachter-Zeh, and E.~Yaakobi.
\newblock The noisy drawing channel: Reliable data storage in dna sequences.
\newblock \emph{IEEE Transactions on Information Theory}, 2022.

\bibitem[Lerman and Maunu(2018)]{lerman2018fast}
G.~Lerman and T.~Maunu.
\newblock Fast, robust and non-convex subspace recovery.
\newblock \emph{Information and Inference: A Journal of the IMA}, 7\penalty0
  (2):\penalty0 277--336, 2018.

\bibitem[Ma et~al.(2021)Ma, Cai, and Li]{ma2021optimal}
R.~Ma, T.~Cai, and H.~Li.
\newblock Optimal permutation recovery in permuted monotone matrix model.
\newblock \emph{Journal of the American Statistical Association}, 116\penalty0
  (535):\penalty0 1358--1372, 2021.

\bibitem[Majumdar and Ward(2011)]{majumdar2011some}
A.~Majumdar and R.~K. Ward.
\newblock Some empirical advances in matrix completion.
\newblock \emph{Signal Processing}, 91\penalty0 (5):\penalty0 1334--1338, 2011.

\bibitem[Marano and Willett(2020)]{marano2020making}
S.~Marano and P.~Willett.
\newblock Making decisions by unlabeled bits.
\newblock \emph{IEEE Transactions on Signal Processing}, 68:\penalty0
  2935--2947, 2020.

\bibitem[Mazumder and Wang(2023)]{Mazumder-MP2023}
R.~Mazumder and H.~Wang.
\newblock Linear regression with partially mismatched data: Local search with
  theoretical guarantees.
\newblock \emph{Mathematical Programming}, 197\penalty0 (2):\penalty0
  1265--1303, 2023.

\bibitem[Muralidhar(2017)]{Muralidhar-JIPS2017}
K.~Muralidhar.
\newblock Record re-identification of swapped numerical microdata.
\newblock \emph{Journal of Information Privacy and Security}, 13\penalty0
  (1):\penalty0 34--45, 2017.

\bibitem[Nejatbakhsh and Varol(2021)]{nejatbakhsh2021neuron}
A.~Nejatbakhsh and E.~Varol.
\newblock Neuron matching in c. elegans with robust approximate linear
  regression without correspondence.
\newblock In \emph{IEEE/CVF Winter Conference on Applications of Computer
  Vision}, 2021.

\bibitem[Oliveira et~al.(2005)Oliveira, Costeira, and
  Xavier]{oliveira2005optimal}
R.~Oliveira, J.~Costeira, and J.~Xavier.
\newblock Optimal point correspondence through the use of rank constraints.
\newblock In \emph{IEEE Computer Society Conference on Computer Vision and
  Pattern Recognition}, 2005.

\bibitem[Onaran and Villar(2022{\natexlab{a}})]{Onaran-arXiv2022}
E.~Onaran and S.~Villar.
\newblock Shuffled linear regression through graduated convex relaxation.
\newblock Technical report, arXiv:2209.15608 [stat.CO], 2022{\natexlab{a}}.

\bibitem[Onaran and Villar(2022{\natexlab{b}})]{onaran2022shuffled}
E.~Onaran and S.~Villar.
\newblock Shuffled linear regression through graduated convex relaxation.
\newblock \emph{arXiv preprint arXiv:2209.15608}, 2022{\natexlab{b}}.

\bibitem[Ongie et~al.(2017)Ongie, Willett, Nowak, and
  Balzano]{ongie2017algebraic}
G.~Ongie, R.~Willett, R.~D. Nowak, and L.~Balzano.
\newblock Algebraic variety models for high-rank matrix completion.
\newblock In \emph{International Conference on Machine Learning}, pages
  2691--2700. PMLR, 2017.

\bibitem[Ongie et~al.(2021)Ongie, Pimentel-Alarc{\'o}n, Balzano, Willett, and
  Nowak]{ongie2021tensor}
G.~Ongie, D.~Pimentel-Alarc{\'o}n, L.~Balzano, R.~Willett, and R.~D. Nowak.
\newblock Tensor methods for nonlinear matrix completion.
\newblock \emph{SIAM Journal on Mathematics of Data Science}, 3\penalty0
  (1):\penalty0 253--279, 2021.

\bibitem[Peng and Tsakiris(2020)]{Peng-SPL2020}
L.~Peng and M.~C. Tsakiris.
\newblock Linear regression without correspondences via concave minimization.
\newblock \emph{IEEE Signal Processing Letters}, 27:\penalty0 1580--1584, 2020.

\bibitem[Peng and Tsakiris(2021)]{Peng-ACHA2021}
L.~Peng and M.~C. Tsakiris.
\newblock Homomorphic sensing of subspace arrangements.
\newblock \emph{Applied and Computational Harmonic Analysis}, 55:\penalty0
  466--485, 2021.

\bibitem[Peng and Vidal(2023)]{Peng-arXiv2023b}
L.~Peng and R.~Vidal.
\newblock Block coordinate descent on smooth manifolds.
\newblock Technical report, arXiv:2305.14744 [math.OC], 2023.

\bibitem[Peng et~al.(2022)Peng, K{\"u}mmerle, and Vidal]{peng2022global}
L.~Peng, C.~K{\"u}mmerle, and R.~Vidal.
\newblock Global linear and local superlinear convergence of irls for
  non-smooth robust regression.
\newblock In \emph{Advances in Neural Information Processing Systems}, 2022.

\bibitem[Rahmani and Atia(2017)]{rahmani2017coherence}
M.~Rahmani and G.~K. Atia.
\newblock Coherence pursuit: Fast, simple, and robust principal component
  analysis.
\newblock \emph{IEEE Transactions on Signal Processing}, 65\penalty0
  (23):\penalty0 6260--6275, 2017.

\bibitem[Ravi et~al.(2022)Ravi, Vahid, and Shomorony]{Ravi-JSAIT2022}
A.~N. Ravi, A.~Vahid, and I.~Shomorony.
\newblock Coded shotgun sequencing.
\newblock \emph{IEEE Journal on Selected Areas in Information Theory},
  3\penalty0 (1):\penalty0 147--159, 2022.

\bibitem[Santa~Cruz et~al.(2017)Santa~Cruz, Fernando, Cherian, and
  Gould]{Cruz-CVPR2017}
R.~Santa~Cruz, B.~Fernando, A.~Cherian, and S.~Gould.
\newblock Deeppermnet: Visual permutation learning.
\newblock In \emph{IEEE Conference on Computer Vision and Pattern Recognition},
  2017.

\bibitem[Santa~Cruz et~al.(2019)Santa~Cruz, Fernando, Cherian, and
  Gould]{Cruz-PAMI2019}
R.~Santa~Cruz, B.~Fernando, A.~Cherian, and S.~Gould.
\newblock Visual permutation learning.
\newblock \emph{IEEE Transactions on Pattern Analysis and Machine
  Intelligence}, 41\penalty0 (12):\penalty0 3100--3114, 2019.

\bibitem[Schmaltz et~al.(2016)Schmaltz, Rush, and Shieber]{Schmaltz-2016}
A.~Schmaltz, A.~M. Rush, and S.~Shieber.
\newblock Word ordering without syntax.
\newblock In \emph{Proceedings of the 2016 Conference on Empirical Methods in
  Natural Language Processing}, pages 2319--2324, Austin, Texas, 2016.
  Association for Computational Linguistics.

\bibitem[Shen et~al.(2017)Shen, Lei, Barzilay, and Jaakkola]{Shen-NeurIPS2017}
T.~Shen, T.~Lei, R.~Barzilay, and T.~Jaakkola.
\newblock Style transfer from non-parallel text by cross-alignment.
\newblock In \emph{Advances in Neural Information Processing Systems}, 2017.

\bibitem[Shomorony and Heckel(2021)]{Shomorony-TIT2021}
I.~Shomorony and R.~Heckel.
\newblock Dna-based storage: Models and fundamental limits.
\newblock \emph{IEEE Transactions on Information Theory}, 67\penalty0
  (6):\penalty0 3675--3689, 2021.

\bibitem[Singer and Cucuringu(2010)]{singer2010uniqueness}
A.~Singer and M.~Cucuringu.
\newblock Uniqueness of low-rank matrix completion by rigidity theory.
\newblock \emph{SIAM Journal on Matrix Analysis and Applications}, 31\penalty0
  (4):\penalty0 1621--1641, 2010.

\bibitem[Slawski and Ben-David(2019)]{Slawski-JoS19}
M.~Slawski and E.~Ben-David.
\newblock Linear regression with sparsely permuted data.
\newblock \emph{Electronic Journal of Statistics}, 13\penalty0 (1):\penalty0
  1--36, 2019.

\bibitem[Slawski et~al.(2020)Slawski, Ben-David, and Li]{Slawski-JMLR2020}
M.~Slawski, E.~Ben-David, and P.~Li.
\newblock Two-stage approach to multivariate linear regression with sparsely
  mismatched data.
\newblock \emph{Journal of Machine Learning Research}, 21\penalty0
  (204):\penalty0 1--42, 2020.

\bibitem[Slawski et~al.(2021)Slawski, Diao, and Ben-David]{Slawski-JCGS2021}
M.~Slawski, G.~Diao, and E.~Ben-David.
\newblock A pseudo-likelihood approach to linear regression with partially
  shuffled data.
\newblock \emph{Journal of Computational and Graphical Statistics}, 30\penalty0
  (4):\penalty0 991--1003, 2021.

\bibitem[Soltanolkotabi and Candes(2012)]{soltanolkotabi2012geometric}
M.~Soltanolkotabi and E.~J. Candes.
\newblock A geometric analysis of subspace clustering with outliers.
\newblock \emph{The Annals of Statistics}, 40\penalty0 (4):\penalty0
  2195--2238, 2012.

\bibitem[Song et~al.(2018)Song, Choi, and Shi]{Song-ISIT18}
X.~Song, H.~Choi, and Y.~Shi.
\newblock Permuted linear model for header-free communication via symmetric
  polynomials.
\newblock In \emph{IEEE International Symposium on Information Theory}, pages
  661--665, 2018.

\bibitem[Sturmfels and Zelevinsky(1993)]{sturmfels1993maximal}
B.~Sturmfels and A.~Zelevinsky.
\newblock Maximal minors and their leading terms.
\newblock \emph{Advances in mathematics}, 98\penalty0 (1):\penalty0 65--112,
  1993.

\bibitem[Tachella et~al.(2023)Tachella, Chen, and Davies]{tachella2022sampling}
J.~Tachella, D.~Chen, and M.~Davies.
\newblock Sensing theorems for unsupervised learning in linear inverse
  problems.
\newblock \emph{Journal of Machine Learning Research}, pages 1--45, 2023.

\bibitem[Tanner and Wei(2013)]{tanner2013normalized}
J.~Tanner and K.~Wei.
\newblock Normalized iterative hard thresholding for matrix completion.
\newblock \emph{SIAM Journal on Scientific Computing}, 35\penalty0
  (5):\penalty0 S104--S125, 2013.

\bibitem[Tsakiris and Vidal(2018{\natexlab{a}})]{Tsakiris-ICML2018}
M.~Tsakiris and R.~Vidal.
\newblock Theoretical analysis of sparse subspace clustering with missing
  entries.
\newblock In \emph{International Conference on Machine Learning},
  2018{\natexlab{a}}.

\bibitem[Tsakiris(2023{\natexlab{a}})]{tsakiris2020results}
M.~C. Tsakiris.
\newblock Results on the algebraic matroid of the determinantal variety.
\newblock Technical report, arXiv:2002.05082v7 [math.AG], 2023{\natexlab{a}}.

\bibitem[Tsakiris(2023{\natexlab{b}})]{tsakiris2023determinantal}
M.~C. Tsakiris.
\newblock Determinantal conditions for homomorphic sensing.
\newblock \emph{Linear Algebra and its Applications}, 656:\penalty0 210--223,
  2023{\natexlab{b}}.

\bibitem[Tsakiris(2023{\natexlab{c}})]{tsakiris2023low}
M.~C. Tsakiris.
\newblock Low-rank matrix completion theory via {P}l{\"u}cker coordinates.
\newblock \emph{IEEE Transactions on Pattern Analysis and Machine
  Intelligence}, 2023{\natexlab{c}}.

\bibitem[Tsakiris and Peng(2019)]{Tsakiris-ICML2019}
M.~C. Tsakiris and L.~Peng.
\newblock Homomorphic sensing.
\newblock In \emph{International Conference on Machine Learning}, 2019.

\bibitem[Tsakiris and Vidal(2015)]{tsakiris2015dual}
M.~C. Tsakiris and R.~Vidal.
\newblock Dual principal component pursuit.
\newblock In \emph{IEEE International Conference on Computer Vision Workshop},
  pages 850--858, 2015.

\bibitem[Tsakiris and Vidal(2017)]{tsakiris2017hyperplane}
M.~C. Tsakiris and R.~Vidal.
\newblock Hyperplane clustering via dual principal component pursuit.
\newblock In \emph{International Conference on Machine Learning}, pages
  3472--3481, 2017.

\bibitem[Tsakiris and Vidal(2018{\natexlab{b}})]{tsakiris2018dual}
M.~C. Tsakiris and R.~Vidal.
\newblock Dual principal component pursuit.
\newblock \emph{Journal of Machine Learning Research}, 19\penalty0
  (1):\penalty0 684--732, 2018{\natexlab{b}}.

\bibitem[Tsakiris et~al.(2020)Tsakiris, Peng, Conca, Kneip, Shi, and
  Choi]{Tsakiris-TIT2020}
M.~C. Tsakiris, L.~Peng, A.~Conca, L.~Kneip, Y.~Shi, and H.~Choi.
\newblock An algebraic-geometric approach for linear regression without
  correspondences.
\newblock \emph{IEEE Transactions on Information Theory}, 66\penalty0
  (8):\penalty0 5130--5144, 2020.

\bibitem[{Unnikrishnan} et~al.(2015){Unnikrishnan}, {Haghighatshoar}, and
  {Vetterli}]{Unnikrishnan-Allerton2015}
J.~{Unnikrishnan}, S.~{Haghighatshoar}, and M.~{Vetterli}.
\newblock Unlabeled sensing: Solving a linear system with unordered
  measurements.
\newblock In \emph{Annual Allerton Conference on Communication, Control, and
  Computing}, 2015.

\bibitem[Unnikrishnan et~al.(2018)Unnikrishnan, Haghighatshoar, and
  Vetterli]{Unnikrishnan-TIT18}
J.~Unnikrishnan, S.~Haghighatshoar, and M.~Vetterli.
\newblock Unlabeled sensing with random linear measurements.
\newblock \emph{IEEE Transactions on Information Theory}, 64\penalty0
  (5):\penalty0 3237--3253, 2018.

\bibitem[Vaswani and Narayanamurthy(2018)]{vaswani2018static}
N.~Vaswani and P.~Narayanamurthy.
\newblock Static and dynamic robust {PCA} and matrix completion: A review.
\newblock \emph{Proceedings of the IEEE}, 106\penalty0 (8):\penalty0
  1359--1379, 2018.

\bibitem[Vaswani et~al.(2018)Vaswani, Bouwmans, Javed, and
  Narayanamurthy]{vaswani2018robust}
N.~Vaswani, T.~Bouwmans, S.~Javed, and P.~Narayanamurthy.
\newblock Robust subspace learning: Robust {PCA}, robust subspace tracking, and
  robust subspace recovery.
\newblock \emph{IEEE Signal Processing Magazine}, 35\penalty0 (4):\penalty0
  32--55, 2018.

\bibitem[{Wang} et~al.(2020){Wang}, {Marano}, {Zhu}, and {Xu}]{Wang-TSP2020}
G.~{Wang}, S.~{Marano}, J.~{Zhu}, and Z.~{Xu}.
\newblock Target localization by unlabeled range measurements.
\newblock \emph{IEEE Transactions on Signal Processing}, 68:\penalty0
  6607--6620, 2020.

\bibitem[Weinberger and Merhav(2022)]{Weinberger-TIT2022}
N.~Weinberger and N.~Merhav.
\newblock The dna storage channel: Capacity and error probability bounds.
\newblock \emph{IEEE Transactions on Information Theory}, 68\penalty0
  (9):\penalty0 5657--5700, 2022.

\bibitem[Xie et~al.(2021)Xie, Mao, Zuo, Xu, Ye, Zhao, and Zha]{Xie-ICLR2021}
Y.~Xie, Y.~Mao, S.~Zuo, H.~Xu, X.~Ye, T.~Zhao, and H.~Zha.
\newblock A hypergradient approach to robust regression without correspondence.
\newblock In \emph{International Conference on Learning Representations}, 2021.

\bibitem[Xu et~al.(2012)Xu, Caramanis, and Sanghavi]{xu2012robust}
H.~Xu, C.~Caramanis, and S.~Sanghavi.
\newblock Robust {PCA} via outlier pursuit.
\newblock \emph{IEEE Transactions on Information Theory}, 58\penalty0
  (5):\penalty0 3047--3064, 2012.

\bibitem[Yang et~al.(2015)Yang, Robinson, and Vidal]{Yang-ICML2015}
C.~Yang, D.~Robinson, and R.~Vidal.
\newblock Sparse subspace clustering with missing entries.
\newblock In \emph{International Conference on Machine Learning}, 2015.

\bibitem[Yao et~al.(2021)Yao, Peng, and Tsakiris]{yao2021unlabeled}
Y.~Yao, L.~Peng, and M.~Tsakiris.
\newblock Unlabeled principal component analysis.
\newblock \emph{Advances in Neural Information Processing Systems}, 2021.

\bibitem[You et~al.(2017)You, Robinson, and Vidal]{you2017provable}
C.~You, D.~P. Robinson, and R.~Vidal.
\newblock Provable self-representation based outlier detection in a union of
  subspaces.
\newblock In \emph{IEEE Conference on Computer Vision and Pattern Recognition},
  2017.

\bibitem[Zeng et~al.(2012)Zeng, Chan, Jia, and Xu]{zeng2012finding}
Z.~Zeng, T.-H. Chan, K.~Jia, and D.~Xu.
\newblock Finding correspondence from multiple images via sparse and low-rank
  decomposition.
\newblock In \emph{European Conference on Computer Vision}, pages 325--339,
  2012.

\bibitem[Zhang and Li(2020)]{Zhang-ICML2020}
H.~Zhang and P.~Li.
\newblock Optimal estimator for unlabeled linear regression.
\newblock In \emph{International Conference on Machine Learning}, pages
  11153--11162, 2020.

\bibitem[Zhang and Yang(2018)]{zhang2018robust}
T.~Zhang and Y.~Yang.
\newblock Robust {PCA} by manifold optimization.
\newblock \emph{Journal of Machine Learning Research}, 19\penalty0
  (1):\penalty0 3101--3139, 2018.

\bibitem[Zhu et~al.(2018)Zhu, Wang, Robinson, Naiman, Vidal, and
  Tsakiris]{zhu2018dual}
Z.~Zhu, Y.~Wang, D.~P. Robinson, D.~Naiman, R.~Vidal, and M.~C. Tsakiris.
\newblock Dual principal component pursuit: Improved analysis and efficient
  algorithms.
\newblock In \emph{Advances in Neural Information Processing Systems}, 2018.

\end{thebibliography}

\end{document}